\documentclass[english]{article}

\usepackage{nicefrac}
\usepackage{cancel}
\usepackage{appendix}

\usepackage[linesnumbered,ruled,vlined]{algorithm2e}

\usepackage{geometry}
\usepackage[T1]{fontenc}
\usepackage[latin9]{inputenc}
\usepackage{bm}
\usepackage{amsmath,mathtools}
\usepackage{amssymb}
\usepackage[unicode=true,
 bookmarks=false,
 breaklinks=false,pdfborder={0 0 1},colorlinks=false]
 {hyperref}
\hypersetup{colorlinks,citecolor=blue,filecolor=blue,linkcolor=blue,urlcolor=blue}

\usepackage{tikz}
  \usetikzlibrary{positioning,fit,backgrounds, matrix,calc}
\usepackage{xcolor}

\makeatletter

\usepackage{footnote}
\makesavenoteenv{algorithm}
 
\usepackage{amsthm}
\usepackage{cite}  
\usepackage{comment}
\usepackage{authblk}
\usepackage[sort,compress]{natbib}
\usepackage{booktabs}
\usepackage{mathabx}
\usepackage{graphicx}
\usepackage{array}

\usepackage{bbding}
\usepackage{colortbl}
\usepackage{makecell}

\usepackage[linesnumbered,ruled,vlined]{algorithm2e}

\SetCommentSty{mycommfont}

\usepackage{float}
\usepackage{multirow}
\usepackage{dsfont}
\usepackage{tcolorbox}
\usepackage{color}
\definecolor{yxc}{RGB}{255,0,0}
\definecolor{ytw}{RGB}{255,69,0}
\definecolor{gen}{RGB}{0,0,200}
\definecolor{zhh}{RGB}{200,200,0}

\allowdisplaybreaks

\makeatletter

\def\LatinUpper{A,B,C,D,E,F,G,H,I,J,K,L,M,N,O,P,Q,R,S,T,U,V,W,X,Y,Z}

\newcommand{\genCal}[1]{\expandafter\newcommand\csname c#1\endcsname{{\mathcal #1}}}
\@for\q:=\LatinUpper\do{%
	\expandafter\genCal\q
}

\newcommand{\genBb}[1]{\expandafter\newcommand\csname b#1\endcsname{{\mathbb #1}}}
\@for\q:=\LatinUpper\do{%
	\expandafter\genBb\q
}

\makeatother

\DeclareMathOperator{\ham}{d_H}

\newcommand{\back}[1]{\overset{\leftarrow}{#1}}

\newcommand{\inflow}{\cF}
\newcommand{\tc}{\mathrm{TC}}
\newcommand{\dtc}{\mathrm{DTC}}
\newcommand{\mi}{\mathrm{I}}
\newcommand{\ent}{\mathcal{H}}

\newcommand{\breg}{D}
\newcommand{\wt}{\widetilde}
\newcommand{\wh}{\widehat}
\newcommand{\defn}{\coloneqq}

\newcommand{\mask}{\mathrm{MASK}}
\newcommand{\qtok}{\cQ}

\renewcommand{\d}{\mathrm{d}}

\newcommand{\KL}{\mathsf{KL}}

\newcommand{\escore}{\varepsilon_{\mathrm{score}}}

\newcommand{\indep}{\perp \!\!\! \perp}

\newtheorem{definition}{Definition}

\theoremstyle{plain} \newtheorem{lemma}{\textbf{Lemma}}\newtheorem{proposition}{\textbf{Proposition}}\newtheorem{theorem}{\textbf{Theorem}}

\theoremstyle{assumption}\newtheorem{assumption}{\textbf{Assumption}}
\theoremstyle{remark}
\theoremstyle{Corollary}\newtheorem{corollary}{\textbf{Corollary}}

\usepackage{enumitem}

\usepackage[nosort,capitalise]{cleveref}
\crefname{enumi}{}{}
\crefname{equation}{Eqn.}{Eqns.}

\crefname{theorem}{Theorem}{Theorems}
\crefname{lemma}{Lemma}{Lemmas}
\crefname{corollary}{Corollary}{Corollaries}
\crefname{proposition}{Proposition}{Propositions}
\crefname{definition}{Definition}{Definitions}
\crefname{remark}{Remark}{Remarks}
\crefname{assumption}{Assumption}{Assumptions}
\crefname{conjecture}{Conjecture}{Conjectures}

\Crefname{theorem}{Theorem}{Theorems}
\Crefname{lemma}{Lemma}{Lemmas}
\Crefname{corollary}{Corollary}{Corollaries}
\Crefname{proposition}{Proposition}{Propositions}
\Crefname{definition}{Definition}{Definitions}
\Crefname{remark}{Remark}{Remarks}
\Crefname{assumption}{Assumption}{Assumptions}
\Crefname{conjecture}{Conjecture}{Conjectures}

\title{Provably adaptive sampling with uniform and remasking \\ discrete diffusion models}

\author{Daniil Dmitriev\thanks{Department of Statistics and Data Science, the Wharton School, University of Pennsylvania; email: \texttt{\{daniild,zhihanh,ytwei\}@wharton.upenn.edu}
},  ~~~Zhihan Huang$^{*}$, ~~~Yuting Wei$^{*}$
}

\date{\today}

\makeatother

\begin{document}

\maketitle 

\begin{abstract}
Discrete diffusion models offer a promising alternative to autoregressive generation by enabling parallel updates, but their sampling efficiency can depend strongly on the choice of the forward process and the sampler. 
For the uniform forward process, existing lower bounds for the standard $\tau$-leaping sampler scale linearly with the ambient dimension $d$, raising the question of whether this dependence is intrinsic to the forward process. We answer this question in the negative.
We consider a first-order sampler based on the leave-one-out denoiser for uniform and remasking processes whose coordinate updates can be performed in parallel. In both cases, the sampler can correct denoising mistakes during the sampling process, which becomes necessary when many coordinates are updated together. Our main result establishes an adaptive sampling guarantee: up to logarithmic factors,
\[
N= O\!\left(\frac{\mathrm{DTC}(X_0)}{\varepsilon}\right)
\]
discretization steps suffice to achieve sampling error $O(\varepsilon_{\mathrm{score}}+\varepsilon)$, where $\varepsilon_{\mathrm{score}}$ is the error in score estimation. 
Thus, the sampling complexity is governed by the intrinsic dependence structure of the target distribution, as measured by its dual total correlation $\mathrm{DTC}(X_0)$, rather than directly by the ambient dimension $d$. In particular, this shows that the unfavorable dimension dependence of $\tau$-leaping for uniform diffusion is a consequence of the sampler rather than the forward process itself. 
Our analysis proceeds through a Bayes-optimal auxiliary sampler that separates discretization error from score-estimation error. We also derive an exact information-theoretic representation of the discretization error in terms of the mutual information between different coordinates of the forward process at different times. This representation applies to general forward processes and, in the uniform and remasking cases, can be controlled by $\mathrm{DTC}(X_0)$. Numerical experiments on structured synthetic distributions illustrate the predicted dimension-adaptive behavior.
\end{abstract}

\setcounter{tocdepth}{2}
\tableofcontents

\section{Introduction}

Diffusion models have achieved remarkable success across a wide range of generative modeling tasks and have recently emerged as a compelling alternative to autoregressive (AR) modeling for discrete sequence generation, such as natural language and protein sequences (\cite{austin2021structured,campbell2022continuous,lou2023discrete}). 
Unlike AR models, which generate tokens sequentially according to a rigid left-to-right factorization, discrete diffusion models enable parallel generation and iterative refinement, making them an increasingly important component of the modern generative modeling toolbox. 

Among existing approaches, masking diffusion has been particularly successful: by progressively replacing data tokens with a dedicated mask token and learning to recover them, masking diffusion has demonstrated strong empirical performance at scale (\cite{sahoo2024simple,shi2024simplified,ou2024your}). 
Nevertheless, the vanilla masking diffusion sampling process has an inherent limitation: unmasking is typically monotone, so once a position is assigned a token, it cannot be revisited in subsequent denoising steps. Therefore, incorrect early predictions may persist and become part of the context used to generate other tokens, potentially leading to error propagation (\cite{xu2025energy,kim2025fine,huang2026don}).
Recent remasking and iterative-refinement methods try to overcome this limitation by explicitly allowing uncertain predictions to be remasked and regenerated, suggesting that the ability to revise intermediate decisions is an important property of discrete generative models (\cite{zhao2026informed,wang2026remasking}). 
Uniform diffusion naturally provides this capability. Since corruption and denoising operate entirely within the original vocabulary, a token can transition between different valid states throughout the sampling trajectory, supporting error correction without relying on a special mask state. Importantly, uniform diffusion is not merely of theoretical interest; recent large-scale systems, including Google DeepMind's DiffusionGemma (\cite{diffusiongemma2026}) and the 7B Sumi model (\cite{ye2026sumi}), adopt uniform-state diffusion, while other recent frameworks, such as GIDD (\cite{von2025generalized}) and XDLM (\cite{liu2026xdlm}), have explored hybrid uniform-masking processes. 
These developments motivate a systematic reconsideration of discrete diffusion beyond masking diffusion as a practical and complementary foundation for large-scale discrete generative modeling.

\subsection{Sampling efficiency and algorithm design}
A central question for discrete diffusion models is their sampling efficiency, namely, how many discretization steps or model evaluations are needed to generate a sample to a prescribed accuracy.
This question is closely tied to the particular choice of the sampler, as different samplers may yield different results. Many samplers were proposed recently, including the \(\tau\)-leaping sampler (\cite{campbell2022continuous}) and its variants (\cite{liang2025discrete}), uniformization sampler~(\cite{chen2024convergence}), DMPM sampler~(\cite{pham2025discrete}), and others. Importantly, certain samplers (e.g., uniformization) only allow one transition at a time. Although theoretically interesting, such a restriction diminishes the main advantage of discrete diffusion models, parallel generation. In what follows, we consider the class of samplers that by design may perform several updates per discretization step, which includes the \(\tau\)-leaping sampler.

Recent work by \cite{dmitriev26efficient} reveals a striking adaptivity property of masking diffusion: the sampling complexity can automatically adapt to the intrinsic dimensionality or structural complexity of the target distribution, 
leading to substantially better efficiency than worst-case guarantees when the target distribution exhibits favorable structure. This echoes a growing literature on continuous diffusion models that establishes analogous forms of adaptation to low-dimensional or structured target distributions; see, e.g., \cite{li2024adapting,li2025dimension,huang2024denoising,liang2025low}.
For uniform discrete diffusion, however, the picture appears less optimistic. Beyond a few special cases, \cite{dmitriev26efficient} show that the sampling complexity of the widely adopted $\tau$-leaping algorithm scales linearly with the ambient dimension \(d\) of the target distribution. Such linear dependence on the ambient dimension essentially precludes fast parallel generation, as AR models also require a number of model evaluations linear in \(d\).
This contrast between the adaptive guarantee for masking diffusion and the lower bound for uniform diffusion raises a fundamental question: 
\begin{center}
\emph{
Is this unfavorable dimension dependence an intrinsic limitation of uniform discrete diffusion, \\or merely a consequence of the $\tau$-leaping sampler?}
\end{center}

In this work, we show that the latter is true and propose a sampler that yields adaptive guarantees for the uniform discrete diffusion. 
Resolving this question is important both for understanding the statistical and computational limits of uniform diffusion and for guiding the design of more efficient sampling algorithms. 
We focus on first-order sampling methods, where each denoising step requires only a single evaluation of the learned model, and investigate the fundamental limits of their sampling efficiency, in contrast to higher-order samplers that use multiple model evaluations per step (e.g.,~\cite{ren2026fast}).

A related open problem concerns remasking discrete diffusion models, 
which have become widely adopted mechanisms for revising previously generated tokens but currently lack a comparable theoretical understanding of their sampling efficiency. Our goal is therefore to characterize when uniform and remasking diffusion can exploit low-dimensional or structured target distributions, and to determine how their sampling complexity can adapt to intrinsic structure in a manner analogous to masking diffusion.

\subsection{An information-theoretic perspective on adaptive sampling}

To understand the adaptive sampling behavior described above, we turn to information-theoretic quantities that capture the intrinsic structure of discrete distributions. As a simple example, consider a $d$-dimensional binary distribution that is uniform over the two strings $0^d$ and $1^d$. While the ambient dimension can be arbitrarily large, the distribution contains only one bit of uncertainty. Moreover, this global structure remains detectable after corruption. Under masking diffusion, observing a single unmasked coordinate determines the original string, whereas under uniform diffusion, the imbalance between the numbers of zeros and ones in the corrupted sequence provides information about the initial state. Thus, the effective difficulty of recovering the underlying sample can remain small even in high dimensions.

This intuition can be formalized using information-theoretic measures of dependence. One such quantity, which has appeared in recent analyses of adaptive sampling (\cite{chen2025optimal,dmitriev26efficient,zhao2026adaptation}), is \emph{the dual total correlation}:
\begin{equation}
    \dtc(X) = \ent(X) - \sum_{i=1}^{d} \ent(X^i \mid X^{-i}),
\end{equation}
where $X=(X^1,\ldots,X^d)$. The DTC measures the dependence among the coordinates that remains after conditioning each coordinate on all the others. Importantly, $\dtc(X)\leq H(X) \leq d \log S$, and there exist high-dimensional distributions whose DTC remains bounded independently of $d$. Hence, guarantees expressed in terms of DTC can be substantially sharper than worst-case bounds that scale directly with the ambient dimension.

More broadly, our analysis is based on an information-theoretic characterization of the sampling error rather than on DTC alone. We show that the discretization error can be expressed through mutual information between different coordinates of the forward process at different times. For the uniform and remasking processes, we shall see that this general characterization can then be controlled in terms of the DTC of the target distribution. 

\begin{figure}[t]
\begin{center}
\includegraphics[width=0.7\linewidth]{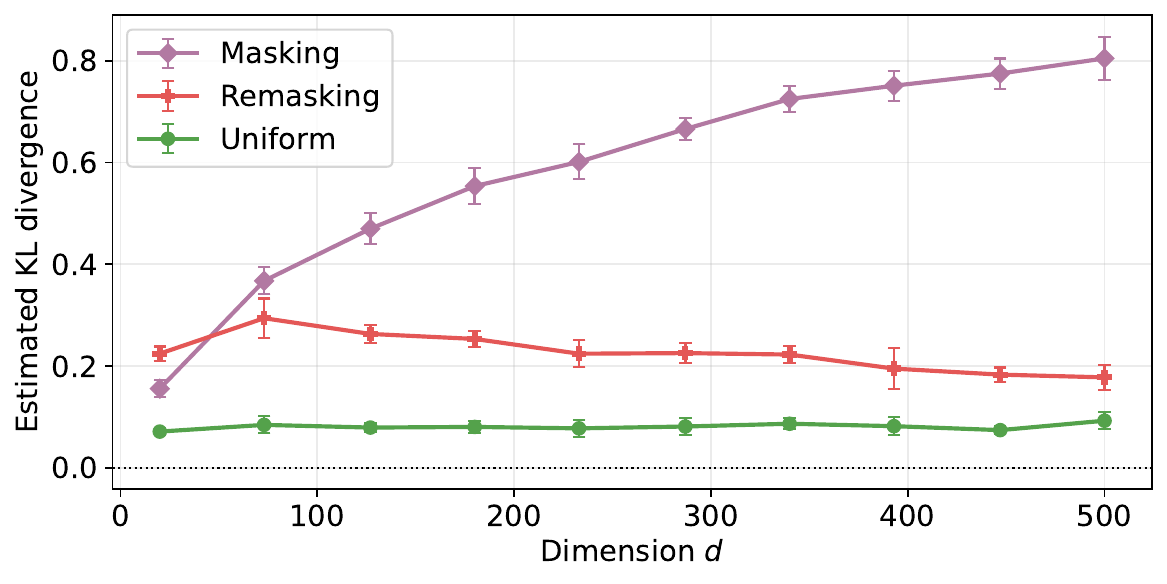}
\end{center}
\caption{\(N = 20\) discretization steps. Target distribution is a Markov chain, see~\Cref{sec:exp}. Remasking process uses \(p_M = 0.5\).  We use early stopping \(\delta = 1\mathrm{e}{-5}\) and time horizon \(T = 8\). 
Different types of discretization schedules are adopted for different noising processes to optimize their performance:
the masking process uses constant discretization \(t_{k+1} - t_k = (T - \delta) / N\); the remasking and uniform processes use geometric discretization \(t_{k+1} - t_k = \kappa \min(1, T - t_{k+1})\) with \(\kappa \approx 1.3\). Results are averaged over 7 runs. 
}
\label{fig:mc-increase-d}
\end{figure}

\subsection{Our main contributions}

Our main contributions are as follows:
\begin{itemize}

\item \textbf{An efficient leave-one-out sampler for uniform and remasking diffusion.}
We study a first-order sampler based on leave-one-out conditional probabilities. On each discretization interval, the resulting approximate reverse process decomposes into independent one-dimensional CTMCs and can be executed for all coordinates in parallel. For the uniform process, our sampler recovers the recently proposed leave-one-out bridge plug-in and cavity ancestral samplers (\cite{gourevitch2026uniform,noguerales2026does}), while our CTMC formulation provides a unified construction that also applies to the remasking process, with the masking process as a special case. 

\item \textbf{Adaptive sampling guarantees beyond ambient dimension.}
Our main result, \Cref{thm:main}, establishes for both the uniform and remasking processes the bound
$$
\KL(q_{T-t_N}\|p_{\mathrm{output}})
\lesssim
\KL(q_T\|q_{\mathrm{noise}})
+\varepsilon_{\mathrm{score}}
+\kappa\dtc(X_0).
$$
Consequently,  it suffices to take 
$
N=\widetilde O\left({\dtc(X_0)}/{\varepsilon}\right),
$
discretization steps to achieve sampling error $O(\varepsilon_{\mathrm{score}}+\varepsilon)$. Thus, the sampling complexity adapts to the intrinsic dependence structure of the target distribution. For the uniform process, this shows that the unfavorable dimension dependence of the standard $\tau$-leaping sampler is not due to the forward process itself, circumventing the lower bound of~\cite{dmitriev26efficient}. Our result also provides an adaptive sampling guarantee for the remasking process, for which theoretical guarantees were previously lacking.

\item \textbf{A Bayes-optimal decomposition of sampling error.}
To disentangle errors arising from discretization and score estimation, we introduce a Bayes-optimal auxiliary sampler that uses the exact leave-one-out conditional probabilities available at each discretization point. The discretization error measures the discrepancy between the true reverse process and the Bayes-optimal sampler, whereas the approximation error measures the additional discrepancy introduced by the learned sampler. The former depends only on the target distribution, forward process, and time discretization, whereas the latter is controlled by the standard score entropy loss and can also be characterized directly through errors in the leave-one-out denoiser. Moreover, in~\Cref{thm:main-discr}, we give an exact information-theoretic representation of the discretization error in terms of mutual information between different coordinates of the forward process at different times. This characterization contains no explicit dependence on ambient parameters such as the dimension or vocabulary size, and may be of independent interest.
\end{itemize}

\subsection{Notation}
For a positive integer \(n\), we denote \([n] = \{1, \ldots, n\}\). We use \(d, S,\) and \(T\) to denote the ambient dimension, the vocabulary size, and the time horizon, respectively. Let \(\cV = [S] \cup \mathrm{Aux}\), where \(\mathrm{Aux}\) is the set of auxiliary states. We have \(\mathrm{Aux} = \emptyset\) or \(\mathrm{Aux} = \{\mask, \mathrm{REMASK}\}\). Let \(q_{\mathrm{data}}\) denote a distribution on \([S]^d\). For \(x = (x^1, \ldots, x^d)\in \cV^d\) and \(i \in [d]\), we denote \(x^{-{i}} \defn (x^1, \ldots, x^{i - 1}, x^{i + 1}, \ldots, x^d) \in \cV^{d - 1}\). For \(x \in \cV^d\), \(i \in [d]\), and \(b \in \cV\), we denote \(x \odot_i b \in \cV^d\) as follows: \((x \odot_i b)^j = x^j\) for \(j \neq i\), and \((x \odot_i b)^i = b\). We use \(\KL, \ent\), and \(\mi\) to denote the KL divergence, the entropy, and the mutual information, respectively. We adopt standard asymptotic notation: \(O(\cdot), o(\cdot), \Omega(\cdot),\) and \(\lesssim\); notation \(\wt O(\cdot)\) hides logarithmic factors in \(d\), \(S\), and \(1 / \varepsilon\). We let \(\breg_{\phi}(x, y) = \phi(x) - \phi(y) - (x - y)^\top \nabla \phi(y) \) be the Bregman divergence for \(\phi: \bR^n \to \bR\), and denote \(\breg(a, b) = \frac a b - 1 - \log \frac a b\) as the Bregman divergence for the scalar function \(\phi(x) = - \log x\).

\subsection{Other related works}

\paragraph{Sampling guarantees for discrete diffusion.} 

A central question in the theory of discrete diffusion models is how many
sampling steps are required to generate an accurate sample. Early work by
\cite{chen2024convergence} studies an exact uniformization-based sampler and
establishes guarantees in both KL divergence and total variation, which eliminates discretization error but requires a number of sampling steps that scales linearly with $d$. 
Subsequent works analyze a broader range of discretized samplers, including
$\tau$-leaping, Euler, and Tweedie $\tau$-leaping
\citep{liang2025discrete}, as well as higher-order schemes
\citep{ren2026fast}. Related convergence guarantees have also been developed
for discrete Markov probabilistic models \citep{pham2025discrete}, absorbing
or masking processes \citep{liang2025absorb,liang2026sharp}, and more general
masked and random-walk dynamics \citep{conforti2025non}.
For the uniform process, \cite{dmitriev26efficient} establish a sharp
$\widetilde O(d/\epsilon)$ complexity for the standard $\tau$-leaping
sampler, together with a matching algorithmic lower bound.

\paragraph{Adaptive guarantees.} 
More recent work has sought sampling guarantees that depend on the intrinsic structure of the target distribution rather than directly on the ambient dimension. 
\cite{libreaking} establish a sampling complexity bound in terms of mutual information, while \cite{chen2025optimal,zhao2026adaptation} sharpen this dependence to information-theoretic quantities such as total correlation and dual total correlation. 
\cite{cai2026confidence} further show that confidence-based unmasking schedules for diffusion language models can achieve sublinear sampling complexity. An interesting recent work by \cite{wainwright2026data} develops an information-theoretic measure of data geometry for masking diffusion, yielding data-dependent sampling guarantees and optimized sampling schedules.
Using a CTMC framework, 
\cite{dmitriev26efficient} establish a guarantee in terms of effective total correlation, which is upper bounded by both total correlation and dual total correlation.
These adaptive guarantees, however, are specific to masking diffusion, either in the diffusion-language-model formulation or in the CTMC framework. In contrast, our results establish adaptive guarantees for the uniform and remasking processes.

\paragraph{Leave-one-out denoisers.} 
Score estimation is a central component of discrete diffusion models, particularly following the score entropy objective introduced by \cite{lou2023discrete}. 
More recently,~\cite{gourevitch2026uniform} and~\cite{noguerales2026does} observe that training the leave-one-out denoiser, or a cavity estimator, is equivalent to score estimation and yields better empirical performance. Our work also adopts the leave-one-out denoiser as a central object, but from a complementary perspective: we show that it arises naturally from the CTMC formulation of the reverse process and use it to construct the sampler and analyze its approximation and discretization errors.

\section{Problem setup}

\subsection{Continuous-time Markov chain}
\label{sec:ctmc}

Let us begin by introducing the framework of discrete diffusion models, which are used to approximate a target distribution \(q_{\mathrm{data}}\) on a discrete domain \([S]^d\).
Analogous to their continuous counterparts, discrete diffusion models consist of a forward process and a reverse process evolving over the discrete state space. Both processes can be formulated in terms of continuous-time Markov chains (CTMCs), which we introduce next, following~\cite{campbell2022continuous}.

\begin{definition}
\label{def:ctmc}
    A stochastic process \((X_t)_{t \in [0, T]}\) on \(\cV^d\) with rate matrices \((Q_t)_{t \in [0, T]}\) and initial distribution \(q_0\) is a \emph{continuous-time Markov chain (CTMC)} if \(X_0 \sim q_0\) and
    \begin{enumerate}
        \item \((X_t)_{t \in [0, T]}\) satisfies the \emph{Markov property}: for any \(0\leq u < v \leq T\), \(X_v\) is conditionally independent of \((X_t)_{t < u}\) given \(X_u\);
        \item As \(\Delta t \downarrow 0\), for any \(x, y \in \cV^d\), \(
            \Pr(X_{t + \Delta t} = y \mid X_t = x) = \bI\{y = x\} + Q_t(x, y)\Delta t + o(\Delta t).\)
    \end{enumerate}
    Here, for any \(t \in [0, T]\), the rate matrix \(Q_t \in \bR^{\cV^d \times \cV^d}\) satisfies:
    \begin{enumerate}
        \item \(Q_t(x, y) \geq 0\) for any \(x \neq y \in \cV^d\),
        \item \(Q_t(x, x) = -\sum_{y \neq x} Q_t(x, y)\).
    \end{enumerate}
\end{definition}
\noindent
For any fixed initial distribution $q_0$, the marginals \((q_t)_{t \in [0, T]}\) of \(X_t\) are the solutions to the Kolmogorov forward equation:
\begin{align*}
    \text{for all } x\in \cV^d, \qquad \frac{\d}{\d t} \Pr(X_t = x) = \sum_{y \in \cV^d} Q_t(y, x) \Pr(X_t = y).
\end{align*}
We refer the reader to \cite{feller1940integro,feinberg2014solutions} for a rigorous treatment of CTMCs.

\paragraph{Forward process.}
We define a forward process as a CTMC \((X_t)_{t \in [0, T]}\) corresponding to the particular choice of rate matrices \((Q_t)_{t \in [0, T]}\) and we assume that there exist \((\qtok_t)_{t \in [0, T]}\), with each \(\qtok_t \in \bR^{\cV \times \cV}\), such that
\begin{enumerate}
    \item \(Q_t(x, y) = 0\) if \(\ham(x, y) \geq 2\),
    \item \(Q_t(x, y) = \qtok_t(x^i, y^i)\), if \(\ham(x, y) = 1\) and \(x^i \neq y^i\).
\end{enumerate}
This requirement states that the forward process acts on each coordinate independently.
In this work, we focus on the following time-homogeneous (\(\qtok_t \equiv \qtok\)) forward processes:

\begin{enumerate}
    \item \emph{Uniform process}: when \(\cV = [S]\) and \(\qtok(a, b) = 1/S\) for all \(a, b \in [S]\), 
    \item \emph{Remasking process}: when \(\cV = [S] \cup \{\mask, \mathrm{REMASK}\}\), and for \(a \in [S]\), \(b \in \cV\), and \(0 < p_M \leq 1\), we have
    \begin{align*}
    \qtok(a, b) = p_M \bI\{b = \mask\} + (1-p_M) \bI\{b = \mathrm{REMASK}\} \quad \text{and} \quad \qtok(\mathrm{REMASK}, a) = 1 / S.
    \end{align*}
    Here, \(\mask\) is an absorbing state and the masking process corresponds to \(p_M = 1\).
\end{enumerate}
For the remasking process, choosing \(p_M < 1\) allows the reverse process to correct unmasked elements, as the coordinates in the forward process will transition through the \(\mathrm{REMASK}\) state one or several times with positive probability. As discussed, this is crucial to mitigate a well-known disadvantage of the masking process, where once a coordinate is unmasked, it cannot be changed later during the sampling process; see~\cite{wang2026remasking,zhao2026informed}.

\paragraph{Self-loops.}
A convenient way to model how the uniform process acts on a single coordinate is as follows: 
the process makes a transition from state \(a \in [S]\) with rate \(1\) to a new state \(b \sim \mathrm{Unif}([S])\), with \(a \indep b\). 
For technical convenience, we define diagonal elements in the uniform case to be \(\qtok(a, a) = 1/S\) for \(a \in [S]\), which makes \(\qtok\) the transition matrix of a discrete-time Markov chain, rather than the rate matrix of a CTMC. As the only difference is in the diagonal elements, we explicitly write \(-\sum_{b \neq a} \qtok(a, b)\) whenever the diagonal entry of the rate matrix is used in the analysis. 
Similarly, for the remasking process, we set \(\qtok(\mask, \mask) = 1\).

\paragraph{Reverse process.}
For the forward process \((X_t)_{t\in [0,T]}\) with its marginal distribution 
\((q_t)_{t\in [0,T]}\), there exists a time-reversed CTMC \((\back X_t)_{t \in [0, T]}\) with an initial distribution \(q_T\) and rate matrices \((\back Q_t)_{t \in [0, T]}:\)
\begin{align*}
    \back Q_t(x, y) = Q_{T - t}(y, x) \frac{\Pr(X_{T - t} = y)}{\Pr(X_{T-t} = x)}, \quad \text{for all } x \neq y \in \cV^d,
\end{align*} 
such that its marginals coincide with the forward process: \(X_t \overset{d}{=} \back X_{T - t}\) for all \(t \in [0, T]\). 
We refer to this CTMC as \emph{the reverse process} (\cite{campbell2022continuous}).
Analogous to the continuous case, where the Stein score function \(\nabla \log p_t(x)\) determines the reverse process,  we define the \emph{(concrete) score function} as the ratio of the probability mass functions 
\begin{align}
    \label{eq:score}
    s_t(y, x) \defn \frac{\Pr(X_t = y)}{\Pr(X_t = x)}. 
\end{align}
As by our construction, \(Q_t(y, x) > 0\) only if \(\ham(x, y) = 1\), we often denote \(y = x \odot_i b\) for some \(i \in [d]\) and \(b \in \cV\).
To construct the reverse process, it is therefore sufficient to know the score functions~\eqref{eq:score}, for every $t \in [0,T]$, \(x \in \cV^d\) and $y = x \odot_i b.$

\subsection{Learning the reverse process}

Computing rate matrices of the reverse process \((\back Q_t)_{t \in [0, T]}\) requires access to score functions at every $t \in [0,T]$, which is not tractable in general. Instead, they are replaced in practice by a data-driven estimator \(\wh s_t(y, x)\) at discrete time points \(0  < T- t_{N - 1} < \ldots < T - t_0 = T\) such that \(\wh s_t(y, x)\approx s_t(y, x) = \Pr(X_t = y) / \Pr(X_t = x)\). The estimated rate matrices are defined as \(\wh Q_t(x, y) = Q_{T - t}(y, x) \wh s_{T - t}(y, x)\).

To measure the accuracy of the estimated score, we use the \emph{score entropy loss} introduced by \cite{lou2023discrete}, which has become a standard objective for training discrete diffusion models. This loss quantifies the discrepancy between the estimated score \(\wh s_t(y, x)\) and the true score $s_t(y, x)$ associated with the forward process:
\begin{align}
\label{eq:se-loss}
    \cL_{\mathrm{SE}}(t, \wh s, s) \defn \bE_{x_t \sim q_t} \left[ \sum_{y \neq x_t} Q_t(y, x_t) s(y, x_t) \breg(\wh s(y, x_t), s(y, x_t))\right].
\end{align}
Here \(t \geq 0\) and \(\wh s, s \in \cV^d \times \cV^d \to \bR_{\geq 0}\).
To disentangle the effect of score estimation from the discretization or optimization errors that govern the efficiency of the sampling algorithms being considered, we isolate the score estimation error and make the following assumption.
\begin{assumption}
\label{asm:se-loss}
    Let \(0 = t_0 < t_1 < \ldots < t_N \leq T\) be the time discretization. We assume that
    \begin{align}
        \sum_{k=0}^{N - 1} (t_{k+1} - t_k) \cL_{\mathrm{SE}}(T - t_k, \wh s_{T - t_k}, s_{T - t_k} )\leq \escore.
    \end{align}
\end{assumption}
This assumption on the score entropy loss is a standard way to control the approximation error of discrete diffusion models and has appeared in several prior works, including~\cite{lou2023discrete,conforti2025non,liang2025discrete,dmitriev26efficient}.

\paragraph{A leave-one-out formulation.} 
We state a leave-one-out denoiser formulation that plays an important role in our analysis and has appeared previously in the context of both remasking~\citep{zhao2026informed} and uniform diffusion~\citep{gourevitch2026uniform}. The key idea is that the estimator must approximate the leave-one-out conditional probabilities \(\Pr(X_0^i = b \mid X_t^{-i} = x^{-i})\) instead of \(\Pr(X_0^i = b \mid X_t = x)\). This can be achieved either by imposing the leave-one-out structure directly through the design of the estimator or by modifying the cross-entropy objective used to train the denoiser. We refer to~\cite{gourevitch2026uniform} for a more detailed discussion.
In our setting, the leave-one-out formulation 
leads to a cleaner characterization of the reverse dynamics and, more importantly, facilitates the control of the resulting sampling errors. This role is different from its use in prior work, where the emphasis is more closely tied to the practical construction and training of the denoiser.

Formally, recall the definition of the score function \(s_t(x \odot_i b, x) = \Pr(X_t = x \odot_i b)\ /\,\Pr(X_t = x)\). Dividing numerator and denominator by \(\Pr(X_t^{-i} = x^{-i})\), we obtain
\begin{align}
\label{eq:score-func}
    s_t(x \odot_i b, x) = \frac{\Pr(X_t = x \odot_i b)}{\Pr(X_t = x)} = \frac{\Pr(X_t^i = b \mid X_t^{-i} = x^{-i})}{\Pr(X_t^i = x^i \mid X_t^{-i} = x^{-i})}.
\end{align}

\noindent
There exists a simple bijection between the sets  \(\left\{ s_t(x\odot_i b, x)\right\}_{b\in\cV}\) and \(\left\{ \Pr(X_t^i = b \mid X_t^{-i} = x^{-i})\right\}_{b\in\cV}\), as
\begin{align*}
    \sum_{b \neq x^i} s_t(x \odot_i b, x) = \frac{\sum_{b \neq x^i} \Pr(X_t^i = b \mid X_t^{-i} = x^{-i})}{\Pr(X_t^i = x^i \mid X_t^{-i} = x^{-i})}=\frac{1}{\Pr(X_t^i = x^i \mid X_t^{-i} = x^{-i})} - 1.
\end{align*}
Indeed, we can use this equality together with the convention \(s_t(x, x) = 1\) to compute for all \(b \in \cV\),
\begin{align*}
\Pr(X_t^i = b \mid X_t^{-i} = x^{-i}) = \frac{s_t(x \odot_i b, x)}{\sum_{c} s_t(x \odot_i c, x) }.
\end{align*}
Therefore, computing the score is equivalent to computing the leave-one-out denoiser probabilities. In the following, in light of this bijection, we refer to both quantities interchangeably.
For both uniform and remasking processes, we find it instrumental to introduce the notation 
\begin{equation}
\label{eq:def-nu}
\nu(t, b) \defn \Pr(X_t^i = b \mid \exists \text{ a jump at the } i\text{-th coordinate on }[0, t]).
\end{equation}
For both uniform and remasking processes, $\Pr(X_t^i = b \mid X_t^{-i} = x^{-i})$ can be expressed in terms of $\nu(t, b)$ and $\Pr(X_0^i = b \mid X_t^{-i} = x^{-i})$:
\begin{equation}
\label{eq:pr-t-from-pr-0}
    \Pr(X_t^i = b \mid X_t^{-i} = x^{-i}) =  \Pr(X_0^i = b \mid X_t^{-i} = x^{-i})e^{-t} +  \nu(t, b)(1 - e^{-t}).
\end{equation}
Indeed, in both cases, for any fixed \(i \in [d]\) and time \(t \geq 0\), either there was a jump at the \(i\)-th coordinate on \([0, t]\), and thus \(X_t^i \indep X_t^{-i}\), or there was no jump and \(X_t^i = X_0^i\), which happens with probability \( e^{-t}\).
\begin{itemize}
\item 
For the uniform process, \(\nu(t, b) = \nicefrac{1}{S}\) and we obtain
\begin{align}
\label{eq:pt-unif}
\Pr(X_t^i = b \mid X_t^{-i} = x^{-i}) =\Pr(X_0^i = b \mid X_t^{-i} = x^{-i}) e^{-t}  + \frac{1}{S}(1 - e^{-t}).
\end{align}

\item For the remasking process,~\Cref{lem:nu-remask} gives the explicit expression for \(\nu(t, b)\). 
\end{itemize}
Observe that the only unknown data-dependent quantities are \(\{\Pr(X_0^i = b \mid X_t^{-i} = x^{-i})\}_{b \in [S]}\).
As a consequence, assuming access to score estimators \(\{\wh s_t(x \odot_i b, x)\}\) is equivalent to assuming access to \(\left\{\wh \Pr(X_0^i = b \mid X_t^{-i} = x^{-i})\right\}_{b\in[S]}\) at the discretization points \(t \in \{T - t_0, \ldots T - t_{N-1}\}\). 

\section{Discrete diffusion sampling with leave-one-out sampler}
\label{sec:sampling}

With the score estimator \(\wh s_t\) that satisfies~\Cref{asm:se-loss},
the remaining task is to construct a tractable approximation to the reverse-time process. Since \(\wh s_t\) is only evaluated at the time discretization \(0 < T - t_{N - 1} < \ldots < T - t_0 = T\), we extend these estimates over $[0,T]$ and use the resulting approximation to simulate the reverse dynamics.
The standard choice, both in theoretical analyses and in practice, is the \emph{\(\tau\)-leaping sampler}~(\cite{campbell2022continuous}) and its variants~(\cite{liang2025discrete}). 
However, recent work by~\cite{dmitriev26efficient} shows that, for the uniform process, the \(\tau\)-leaping sampler can lead to suboptimal sampling complexity. In this section, we study the following leave-one-out sampler which leads to improved sampling efficiency and overcomes the theoretical barrier of $\tau$-leaping. 

\subsection{Leave-one-out sampler}
Consider the \(k\)-th discretization interval and let \(u = T - t_{k+1}\) and \(\ell = T - t_k\).
As discussed above, we assume that the sampling algorithm has access to \(\{\wh \Pr(X_0^i = b \mid X_{\ell}^{-i} = x_{\ell}^{-i})\}\) for \(b \in [S]\) and \(i\in[d]\),
and our goal is to sample \(x_{u}\) given \(x_{\ell}\) and \(\{\wh \Pr(X_0^i = b \mid X_{\ell}^{-i} = x_{\ell}^{-i})\}_{b\in[S], i\in[d]}\).
Following~\Cref{eq:pr-t-from-pr-0}, we define \( \wh \Pr(X_t^i = b \mid X_{\ell}^{-i} = x_{\ell}^{-i})\) as follows:
\begin{align}
\label{eq:whpt-remask}
    \wh \Pr(X_t^i = b \mid X_{\ell}^{-i} = x_{\ell}^{-i}) = \wh \Pr(X_0^i = b \mid X_{\ell}^{-i} = x_{\ell}^{-i}) e^{-t} + \nu(t, b)(1 - e^{-t}).
\end{align}
Together with~\Cref{eq:score}, we define our score estimator~\(\wh s_t\) for \(t \in [u, \ell]\):
\begin{align}
\label{eq:our-shat}
    \wh s_t(x \odot_i b, x) \defn \frac{ \wh \Pr(X_t^i = b \mid X_{\ell}^{-i} = x_{\ell}^{-i})}{\wh \Pr(X_t^i = x^i \mid X_{\ell}^{-i} = x_{\ell}^{-i})} = \frac{\wh \Pr(X_0^i = b \mid X_{\ell}^{-i} = x_{\ell}^{-i}) e^{-t} + \nu(t, b)(1 - e^{-t})}{\wh \Pr(X_0^i = x^i \mid X_{\ell}^{-i} = x_{\ell}^{-i}) e^{-t} + \nu(t, x^i)(1 - e^{-t})},
\end{align}
and the corresponding rate matrix \(\wh Q_t(x, y) = Q_{T - t}(y, x)\wh s_{T - t}(y, x)\). Note that both \(\wh s_t(x \odot_i b, x)\) and \(\wh Q_t(x, x \odot_i b)\) do not depend on \(x^{-i}\). This means that effectively on the discretization interval \([u, \ell]\), the approximate CTMC can be decomposed into \(d\) independent one-dimensional CTMCs, allowing us to simulate the dynamics of the CTMC in parallel for each fixed discretization interval, as is done in~\Cref{alg:sampler}, see~\Cref{prop:sampler}.

In view of the Bayes formula, we can write
\begin{align*}
    \Pr(X_{u}^i = b \mid X_{\ell}=  x_{\ell}) &= \frac{\Pr(X_{\ell}^i = x_{\ell}^i \mid X_u^i = b) \Pr(X_{u}^i =b \mid X_{\ell}^{-i} = x_{\ell}^{-i})}{\Pr(X_{\ell}^i = x_{\ell}^i \mid X_{\ell}^{-i} = x_{\ell}^{-i})} \\
    &\propto\ \Pr(X_{\ell}^i = x_{\ell}^i \mid X_u^i = b)\Pr(X_u^i = b \mid X_{\ell}^{-i} = x_{\ell}^{-i}).
\end{align*}
Here, \(\Pr(X_{\ell}^i = x_{\ell}^i \mid X_u^i = b)\) does not depend on \(q_{\mathrm{data}}\) and can be computed from the properties of \(\qtok\), and 
\begin{align*}\Pr(X_{u}^i =b \mid X_{\ell}^{-i} = x_{\ell}^{-i}) = \underbrace{\Pr(X_0^i = b \mid X_{\ell}^{-i} = x_{\ell}^{-i}) e^{-u}}_{\text{no jump on }[0, u]} + \underbrace{\nu(u, b)(1 - e^{-u})}_{\text{jump on }[0, u]},
\end{align*}
where \(\nu(u, b)\) is defined in~\Cref{eq:def-nu}.
Therefore, the sampling algorithm applies the following in parallel over \(i \in [d]\):
Given $x_{\ell}$, sample $x_u^i$ from  $\wh \mu_i$,
\begin{align}
\label{eq:sampling-alg}
     \wh \mu_i(b)\ \propto\ \Pr(X_{\ell}^i = x_{\ell}^i \mid X_u^i = b)\left(\wh \Pr(X_0^i = b \mid X_{\ell}^{-i} = x_{\ell}^{-i}) e^{-u} + \nu(u, b)(1 - e^{-u})\right).
\end{align}
Finally, we remark that as \(\Pr(X_\ell^i = a \mid X_u^i = b)\) depends only on the matrix \(\qtok\), it can be explicitly computed. 
Thus, we set 
\begin{align}
\label{eq:whpr-eq-pr}
   \wh \Pr(X_\ell^i = a \mid X_u^i = b) = \Pr(X_\ell^i = a \mid X_u^i = b),
\end{align}
for all \(0 \leq u \leq \ell\) and \(a, b \in \cV\).

Putting things together and instantiating the above to the processes that we consider gives:
\begin{itemize}
\item For the uniform process,
\begin{align*}
    \wh \mu_i(b) = \frac{ \left(\bI\{b = x_{\ell}^i\} e^{-(\ell - u)} + \frac{1 - e^{-(\ell - u)}}{S}\right)\left(\wh \Pr(X_0^i = b \mid X_{\ell}^{-i} = x_{\ell}^{-i}) e^{-u} +  \frac{1 - e^{-u}}{S}\right)}{\wh \Pr(X_0^i = x_{\ell}^i \mid X_{\ell}^{-i} = x_{\ell}^{-i}) e^{-\ell} + 
\frac{1 - e^{-\ell}}{S}}.
\end{align*}
    \item 
For the masking process, for \(i \in [d]\) with \(x_{\ell}^i = \mask\),
\begin{align*}
    \wh \mu_i(b) = \frac{e^{-u} - e^{-\ell}}{1 - e^{-\ell}} \wh \Pr(X_0^i = b \mid X_{\ell}^{-i} = x_{\ell}^{-i}) \quad \text{for } b\in[S], \qquad \text{and} \qquad \wh \mu_i(\mask) = \frac{1 - e^{-u}}{1 - e^{-{\ell}}},
\end{align*}
recovering Algorithm 1 from~\cite{dmitriev26efficient}.
\item For the general remasking process,~\Cref{lem:sampling-alg-exact-remask} gives explicit expressions.
\end{itemize}

The next proposition provides a simple algorithm to simulate this CTMC for both uniform and remasking processes.
The proof of this result is deferred to Section~\ref{sec:pf-prof-sampler}.

\begin{algorithm}[t]
\SetAlgoLined
\DontPrintSemicolon 
\caption{Our sampling algorithm}
\label{alg:sampler}

\SetKwInput{KwInput}{Input}
\SetKwInput{KwOutput}{Output}

\KwInput{\\
Initial distribution: \(p_0\), \\
Discretization steps: \(0 = t_0 < t_1 < \ldots < t_N \leq T\), \\
Leave-one-out denoiser: \(\wh \Pr(X_0^i = b \mid X_{T - t}^{-i} = \cdot\ )\) for \(t \in \{t_0, \ldots, t_{N - 1}\}\), \(b \in [S]\), and \(i \in [d]\).}
\KwOutput{Sample \(\wh x \in \cV^d\).}

Sample \(x_T\) from \(p_0\)

\For{\(k = 0, \ldots, N - 1\)}{
    \(u \gets T - t_{k+1}\)\;
    \(\ell \gets T - t_k\)\;
    \For{\(i \in [d]\) \emph{in parallel} }{
    Let \(\wh \mu_i\) be a probability distribution over \(\cV\) with \(\wh \mu_i(b)\ \propto\ \Pr(X_{\ell}^i = x_{\ell}^i \mid X_u^i = b)\left(\wh \Pr(X_0^i = b \mid X_{\ell}^{-i} = x_{\ell}^{-i})e^{-u} + \nu(u, b) (1 - e^{-u})\right)
    \)
        
        Sample \(x_u^i\) from \(\wh \mu_i\)\;

    }
}
\Return{$x_{T - t_N}$}
\end{algorithm}

\begin{proposition}
\label{prop:sampler}
    Fix \(k \in \{0, \ldots, N - 1\}\) and let \(u = T - t_{k+1}, \ell = T - t_{k}\). 
    Let \(x_{\ell}\) and \(x_{u}\) be as in~\Cref{alg:sampler}. Let \(y_{u}\) be the distribution of a CTMC initialized at \(x_{\ell}\) with the rate matrices \(\wh Q_t\) defined in~\Cref{eq:our-shat}. Then, 
    $$x_{u} \overset{d}{=} y_{u}.$$ Consequently,~\Cref{alg:sampler} simulates the full dynamics of the CTMC with initial distribution \(p_0\) and rate matrices \((\wh Q_t)_{t \in [0, t_N]}\).
\end{proposition}

For the uniform process, our sampler coincides with prior work, e.g., the leave-one-out bridge plug-in sampler of \cite{gourevitch2026uniform} after identifying their noise schedule with \(\alpha_t=e^{-t}\). The two approaches, however, arise from different perspectives. \cite{gourevitch2026uniform} derive the leave-one-out predictor as the optimal target for the bridge plug-in parameterization and study its implications for training and inference. In contrast, we derive the same sampling transition directly from the CTMC formulation: on each discretization interval, we condition the leave-one-out denoiser on the state available at the beginning of the interval and construct an approximate reverse-time CTMC, which decomposes into independent one-dimensional processes that can be simulated in parallel. This viewpoint provides a natural Bayes-optimal intermediate process for separating approximation and discretization errors and, moreover, extends within a unified framework to both the uniform and remasking processes.

\subsection{Bayes-optimal sampler}
To separate the error due to score estimation from the error due to time discretization, we introduce an oracle counterpart of our sampler. On each discretization interval \([u,\ell]\), this oracle has access to the exact leave-one-out conditional probabilities at the beginning of the interval, but, like our practical sampler, it does not observe the evolving context $X_t^{-i}$ for $t\in[u,\ell]$. Thus, it provides a natural intermediate process between the true reverse process and our approximation. More precisely, we define the following. 
\begin{definition}
\label{def:wt-score}
Fix \(k \in \{0, \ldots, N - 1\}\) and let \(u = T - t_{k+1},\ \ell = T - t_k\). For \(t \in [u, \ell]\), define
    \begin{align}
        \wt s_t(x \odot_i b, x) = 
        \frac{\Pr(X_t^i = b \mid X_{\ell}^{-i} = x_{\ell}^{-i})}{\Pr(X_t^i = x^i \mid X_{\ell}^{-i} = x_{\ell}^{-i})},
    \end{align}
    and the corresponding rate matrix by: \(\wt Q_t(x, y) = Q_{T - t}(y, x) \wt s_{T-t}(y, x)\). 
\end{definition}
Equivalently, $\widetilde{s}_t$ is obtained from our score estimator $\widehat{s}_t$ by replacing the estimated leave-one-out probabilities at the discretization point with their population counterparts.

The following lemma justifies that this term is indeed Bayes optimal: among estimators restricted to the information available at the beginning of the discretization interval, $\wt {Q}_t$ is the conditional expectation of the true reverse rate.

\begin{lemma}
\label{prop:cond-exp}
Fix \(k \in \{0, \ldots, N - 1\}\). Let \(t \leq \ell \defn T - t_k\).
For any \(x_{\ell}^{-i} \in \cV^{d - 1}\) and \(x_t^i, b \in \cV\), it holds that 
    \begin{align*} 
    \wt Q_{T -t}(x \odot_i x_t^i, x \odot_i b) =
        \bE_{x_t^{-i} \sim q_{t}^{-i}} \Bigl[\back Q_{T - t}(x_t, x_t \odot_i b)\ \Bigr\rvert\ X_{\ell}^{-i} = x_{\ell}^{-i}, X_t^i = x_t^i\Bigr].
    \end{align*}
\end{lemma}

\begin{table}[t]
\centering
\renewcommand{\arraystretch}{1.5}
\begin{tabular}{lccc}
\toprule
& \textbf{True reverse process} & \textbf{Bayes-optimal sampler} & \textbf{Our approximation} \\
\midrule
\textbf{Score function} & \(\displaystyle\frac{\Pr(X_t^i = b \mid X_t^{-i} = x_t^{-i})}{\Pr(X_t^i = x_t^i \mid X_t^{-i} = x_t^{-i})}\) & \(\displaystyle\frac{\Pr(X_t^i = b \mid X_{\ell}^{-i} = x_{\ell}^{-i})}{\Pr(X_t^i = x_t^i \mid X_{\ell}^{-i} = x_{\ell}^{-i})}\) & \(\displaystyle\frac{\wh \Pr(X_t^i = b \mid X_{\ell}^{-i} = x_{\ell}^{-i})}{\wh \Pr(X_t^i = x_t^i \mid X_{\ell}^{-i} = x_{\ell}^{-i})}\) \\[1em]
\midrule
\textbf{LOO denoiser} & \(\Pr(X_0^i = b \mid X_t^{-i} = x_t^{-i})\) & \( \Pr(X_0^i = b \mid X_{\ell}^{-i} = x_{\ell}^{-i})\) &  \(\wh \Pr(X_0^i = b \mid X_{\ell}^{-i} = x_{\ell}^{-i}) \) \\
\bottomrule
\end{tabular}
\caption{Here, \(\ell = T - t_{k}\) is the end point of the \(k\)-th discretization interval and \(T - t_{k+1} \leq t \leq \ell\). The score function row compares the true value \(s_t(x_t \odot_i b, x_t)\), the Bayes-optimal \(\wt s_t(x_t \odot_i b, x_t)\), and ours \(\wh s_t(x_t \odot_i b, x_t)\). The reverse process conditions on the correct context at time \(t\). The Bayes-optimal estimator conditions on all information available at the start of the discretization step, i.e., at time \(\ell\). Our approximation is based on the Bayes optimal one, but instead of the true probabilities \(\Pr(\cdot)\) it uses the estimate \(\wh \Pr(\cdot)\), similarly for the leave-one-out  denoiser. To compute the probability of \(X_t^i = b\) used in the score function, all three approaches use the same formula (see~\Cref{eq:pr-t-from-pr-0,eq:whpt-remask}) but with different leave-one-out denoisers.}
\label{tab:methods_comparison}
\end{table}

\noindent
We refer to~\Cref{tab:methods_comparison} for the comparison of the true score function \(s_t\), Bayes-optimal \(\wt s_t\), and the approximation \(\wh s_t\) that we use. 
We emphasize that, for the time \(t \leq \ell = T - t_k\) and coordinate \(i \in [d]\), the true reverse process evaluates the score function using the full current context $X_t^{-i}$, while the Bayes-optimal sampler and our sampler can only access the context \(X_{\ell}^{-i}\) available at the beginning of the interval. Our practical sampler introduces one additional approximation by replacing the exact conditional probabilities with their learned estimates compared to the Bayes-optimal sampler. 

In contrast, the score function used by the standard \(\tau\)-leaping sampler~(\cite{campbell2022continuous}) is given by:
\begin{equation*}
    s_t^{\tau}(x_t \odot_i b, x_t) = \frac{\wh \Pr(X_\ell^i = x_\ell^i  + (b - x_t^i)\mid X_{\ell}^{-i} = x_{\ell}^{-i})}{\wh \Pr(X_\ell^i = x_\ell^i \mid X_{\ell}^{-i} = x_{\ell}^{-i})}.
\end{equation*}
To evaluate the likelihood of transition \(x_t^i \to b\), the \(\tau\)-leaping sampler uses transition \(x_{\ell}^i \to x_{\ell}^i + (b - x_t^i)\) at the beginning of the discretization interval. This construction implicitly relies on an ordinal structure of the state space, an issue already noted in~\cite{campbell2022continuous}, where truncation was proposed for non-ordered vocabularies.
Moreover, \(\tau\)-leaping ignores the time index \(t\), and computes probabilities with respect to the beginning of the discretization interval \(\ell\). Both our sampler and the Bayes-optimal sampler mitigate these drawbacks.
\section{Main results}

We now present our main sampling guarantees for the uniform and remasking diffusion models. 
Our result shows that, under suitable choices of discretization, the sampling complexity is governed by the intrinsic dependence structure of the target distribution, as measured by its dual total correlation. 
All proofs for the results in this section are given in~\Cref{sec:main-proofs}.

\begin{theorem}
\label{thm:main}
Let \(0 = t_0 < t_1 < \ldots < t_N < T\) and suppose that for some \(\kappa \in (0, 1)\),  \(t_{k+1} - t_k \leq \kappa \min(1, T - t_{k+1})\) for all \(k \in \{0, \ldots, N - 1\}\). 
 Consider either of the following two processes:
     \begin{enumerate}
         \item \emph{Uniform}: \((q_t)_{t \in [0, T]}\) are the marginals of the uniform process,
         \item \emph{Remasking}: \((q_t)_{t \in [0, T]}\) are the marginals of the remasking process.
     \end{enumerate}
     Under~\Cref{asm:se-loss},~\Cref{alg:sampler} initialized from \(p_0 = q_{\mathrm{noise}}\) outputs a sample \(x_{T - t_N} \sim  p_{\mathrm{output}}\) such that
    \begin{align}
        \label{eq:thm-kl-decomp}
        \KL(q_{T - t_N} \| p_{\mathrm{output}}) \lesssim \KL(q_T \| q_{\mathrm{noise}})+ \escore + \kappa\dtc(X_0).
    \end{align}
    In particular, for \(T =  O\left(\log (\varepsilon^{-1}d \log S)\right)\), and corresponding choice of \(q_{\mathrm{noise}}\)\footnote{\(q_{\mathrm{noise}} = \mathrm{Unif}\{[S]^d\}\) for the uniform process and \(q_{\mathrm{noise}} = \mu^{\otimes d}\) with \(\mu(b) = \bI\{b \in [S]\} \nicefrac{e^{-T}  }{S} + \nu(T, b) (1 - e^{-T})\) for the remasking process.}, it suffices to take 
    \begin{equation*}
        N = \wt O\left(\frac{\dtc(X_0)}{\varepsilon}\right),
    \end{equation*}
    discretization steps to guarantee \(\KL(q_{T - t_N} \| p_{\mathrm{output}}) \lesssim \escore + \varepsilon\).
\end{theorem}

This result separates the sampling error into three sources: initialization error, score-estimation error, and an intrinsic discretization error controlled by the dependence structure of the target distribution. In particular, the resulting step complexity scales with $\dtc(X_0)$ rather than explicitly with the ambient dimension $d$. For structured high-dimensional distributions with $\dtc(X_0)\ll d$, this can yield a substantially sharper guarantee than dimension-dependent worst-case bounds; see concrete examples in \cite{dmitriev26efficient}.

For the uniform process, this result shows that the unfavorable dimension dependence previously established (\cite{dmitriev26efficient}) for the standard $\tau$-leaping sampler is not intrinsic to the forward process itself, but can instead arise from the choice of sampling algorithm. The same analysis also yields an adaptive sampling guarantee for the remasking process, for which, to the best of our knowledge, no comparable theoretical guarantee was previously available.

\begin{proof}[Proof sketch]
Using the Bayes-optimal sampler introduced in~\Cref{sec:sampling}, we decompose the KL divergence between two path measures into approximation and discretization errors, see~\Cref{thm:main-err-decomp}. To upper bound the discretization error, we express it as the integral of the second partial derivative of the mutual information (\Cref{thm:main-discr}) and then upper bound it for both considered forward processes (\Cref{prop:main-discr}). The proof of the latter result is based on Gr\"{o}nwall's inequality.
\end{proof}

The following proposition formalizes the specific error decomposition underlying this argument. It is developed for general forward processes by comparing them with auxiliary CTMCs, and therefore separates the general information-theoretic part of our analysis from the process-specific bounds developed later. The proof of this result is included in Section~\ref{sec:pf-decomposition}.

\begin{proposition}
\label{thm:main-err-decomp}
Let \(0 = t_0 < t_1 < \ldots < t_N \leq T\) be the time discretization. Recall \(s_t\), \(\wh s_t\) from~\Cref{eq:score-func,eq:our-shat}, and \(\wt s_t\) from~\Cref{def:wt-score}. Then,
    \begin{equation*}
       \KL(q_{T - t_N} \| p_{\mathrm{output}}) \leq \KL(q_T \| q_{\mathrm{noise}}) + \sum_{k=0}^{N -1}\cL_{\mathrm{approx}}^{(k)} + \sum_{k=0}^{N - 1}\cL_{\mathrm{discr}}^{(k)}, 
    \end{equation*}
    where, for \(k \in \{0, \ldots, N - 1\}\), \(u = T - t_{k+1}\), and \(\ell = T - t_k\),
    \begin{equation}
    \label{eq:def-approx}
    \cL_{\mathrm{approx}}^{(k)} \defn \int_u^{\ell} \bE_{x_t, x_\ell \sim q_{t, \ell}} \left[ \sum_{y \neq x_t} Q_t(y, x_t) \wt s_t(y, x_t) \breg(\wh s_t(y, x_t), \wt s_t(y, x_t))\right] \d t
    \end{equation}
    and
    \begin{equation}
    \label{eq:def-discr}
    \cL_{\mathrm{discr}}^{(k)} \defn \int_u^{\ell} \bE_{x_t, x_\ell \sim q_{t, \ell}} \left[ \sum_{y \neq x_t} Q_t(y, x_t)  s_t(y, x_t) \breg(\wt s_t(y, x_t), s_t(y, x_t))\right] \d t.
    \end{equation}
\end{proposition}

Proposition~\ref{thm:main-err-decomp} provides a general and interpretable decomposition of the sampling error. A key feature of this decomposition is that the discretization error is independent of the particular sampler: it measures the discrepancy between the true reverse process and the Bayes-optimal sampler, and therefore depends only on the target distribution, the forward process, and the chosen time discretization. In contrast, the approximation error measures the discrepancy between a particular sampler and its Bayes-optimal counterpart, thereby isolating the error arising from approximating the reverse dynamics. This separation allows the discretization error to be studied independently of sampler-specific approximations, and the resulting analysis applies beyond Algorithm~\ref{alg:sampler}, including the $\tau$-leaping sampler and its variants.

\paragraph{Discretization error.}

We first characterize the discretization error independently of the particular forward process. The following theorem provides an exact information-theoretic representation in terms of the mutual information between one coordinate and the remaining coordinates at different times. The proof is deferred to~\Cref{sec:discr-proof}.

\begin{theorem}
\label{thm:main-discr}
Fix \(k \in \{0, \ldots, N - 1\}\), let \(u = T - t_{k+1}\) and \(\ell = T - t_k\), and recall \(\cL_{\mathrm{discr}}^{(k)}\) from~\Cref{eq:def-discr}.
Then,
    \begin{equation}
         \label{eq:discr-err-exact}\cL_{\mathrm{discr}}^{(k)} = \sum_{i\in[d]}\int_{u}^{\ell} \int_t^{\ell} \frac{\partial^2}{\partial t \partial v}\mi(X_t^i \ ;\,X_v^{-i})\d v \d t.
    \end{equation}
\end{theorem}

We emphasize that~\Cref{thm:main-discr} holds \emph{for any} forward process and shows that the discretization error does not explicitly depend on the problem parameters, such as ambient dimension \(d\) and vocabulary size \(S\), but instead on the information-theoretic properties of both the data distribution and the forward process. We next specialize this characterization to the uniform and remasking processes, and leave potential applications for other cases, e.g., discrete Gaussian, or semantic-dependent (see, e.g.,~\cite{austin2021structured}), for future work.

\begin{proposition}
\label{prop:main-discr}
    Fix \(k \in \{0, \ldots, N - 1\}\) and let \((X_t)_{t \in [0, T]}\) be the uniform or remasking process. For \(u = T - t_{k+1}\) and \(\ell = T - t_k\) with \(\ell - u \leq \kappa \min(1, u)\), we have
    \begin{equation*}
         \cL_{\mathrm{discr}}^{(k)} \lesssim \kappa \left(\dtc(X_{u}) - \dtc(X_{\ell})\right).
    \end{equation*}
\end{proposition}

Using a telescoping sum,~\Cref{prop:main-discr} immediately leads to \(\sum_{k=0}^{N - 1} \cL_{\mathrm{discr}}^{(k)} \leq \kappa \dtc(X_0)\). We remark that while the statement requires \(\ell - u \leq \kappa \min(1, u)\), and thus imposes a geometric discretization grid \(t_{k+1} - t_k \leq \kappa \min(1, T - t_{k+1})\), the discretization error on the interval \([0, \delta]\) can  also be made small, for \(\delta\) small enough. Therefore, the early stopping requirement in~\Cref{thm:main} arises from the control of the approximation error rather than the discretization error, which we shall discuss next.

\paragraph{Approximation error.}

We next study the approximation error \(\cL_{\mathrm{approx}}\) defined in~\Cref{eq:def-approx}, which measures the discrepancy between a practical sampler and its Bayes-optimal counterpart.
We first show that it is directly controlled by the score entropy loss appearing in~\Cref{asm:se-loss}. 
The proofs for the following two results are deferred to~\Cref{sec:approx-proof}.

\begin{proposition}
\label{lem:approx-to-se}
Consider the \(k\)-th interval of the time discretization \(0 = t_0 < t_1 < \ldots < t_N < T\) and let \(u = T - t_{k+1}\), \(\ell = T - t_k\). Recall the true score \(s_t\) from~\Cref{eq:score} and our score estimator \(\wh s_t\) from~\Cref{eq:our-shat}. If \(\ell - u \leq \kappa \min(1, u)\) with \(0 < \kappa < 1\), then
    \begin{equation}
        \cL_{\mathrm{approx}}^{(k)} \lesssim  (\ell - u) \cL_{\mathrm{SE}}(\ell, \wh s_{\ell}, s_{\ell}).
    \end{equation}
\end{proposition}

Let us next give an alternative characterization directly in terms of the leave-one-out denoiser. This formulation makes explicit how errors in estimating the conditional distribution of $X_0^i$ from the context $X_t^{-i}$ contribute to the sampling error. Importantly, the coordinate $X_t^i$ itself is excluded from the conditioning context.
\begin{theorem}
\label{thm:main-approx}
   Let \((X_t)_{t \in [0, T]}\) be the CTMC corresponding to the uniform or remasking process. If  \(t_{k+1} - t_k \leq \kappa \min(1, T - t_{k+1})\) for \(k \in \{0, \ldots, N - 1\}\), then
    \begin{equation*}
        \sum_{k=0}^{N - 1} \cL_{\mathrm{approx}}^{(k)} \lesssim \kappa \sum_{k=0}^{N - 1}  \sum_{i \in [d]} e^{-(T - t_{k+1})} \KL \left( \mu_{X_0^i} \ \big\|\ \wh\mu_{X_0^i} \ \Bigr|\  X_{T - t_k}^{-i}\,\right).
    \end{equation*}
    Here, for fixed \(i \in [d]\), \(\KL \left( \mu_{X_0^i} \ \big\|\ \wh\mu_{X_0^i} \ \Bigr|\  X_{T - t_k}^{-i}\,\right)\) is the conditional KL divergence between \(\Pr(X_0^i = \, \cdot \mid X_{T - t_k}^{-i})\) and \(\wh \Pr(X_0^i = \, \cdot \mid X_{T - t_k}^{-i})\).
\end{theorem}

\Cref{thm:main-approx} gives a quantitative bound for the approximation error with respect to the cross-entropy loss (\cite{austin2021structured,sahoo2024simple}). The cross-entropy loss is widely used in practice for training discrete diffusion models~(e.g.,~\cite{team2026diffusiongemma}), although not in the leave-one-out formulation. This result may be of future interest, e.g., for proving bounds on the sample complexity of score estimation.
\section{Numerical examples}
\label{sec:exp}

\begin{figure}[t]
\includegraphics[width=1.0\linewidth]{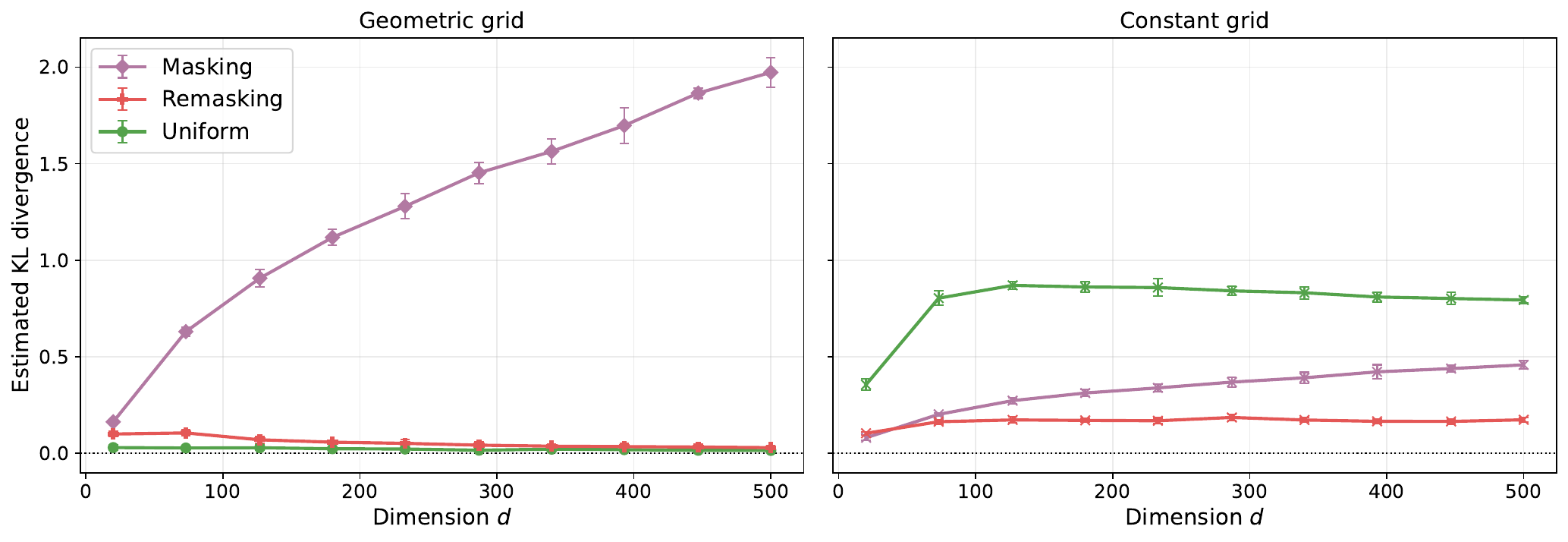}
\caption{\(N = 30\) discretization steps. Target distribution is the Markov chain with \(d p_{\mathrm{flip}} = 2\). The left plot shows the performance on the geometric grid \(t_{k+1} - t_k = \kappa \min(1, T - t_{k+1})\) and the right plot shows the result for the constant grid \(t_{k+1} - t_k = (T-\delta) / N\). We use early stopping at \(\delta = 1\mathrm{e}{-5}\) and \(T = 8\). Results are averaged over \(7\) runs. Given the same number of discretization steps \(N\), the masking process works better with the constant grid, which agrees with the finding in~\cite{dmitriev26efficient} for distributions with small \(\dtc\). Uniform process, in contrast, works better with the geometric grid, consistent with the finding of this work. Remasking process performs robustly on both grids.} 
\label{fig:compare-grid}
\end{figure}

\begin{figure}[t]
\begin{center}
\includegraphics[width=0.7\linewidth]{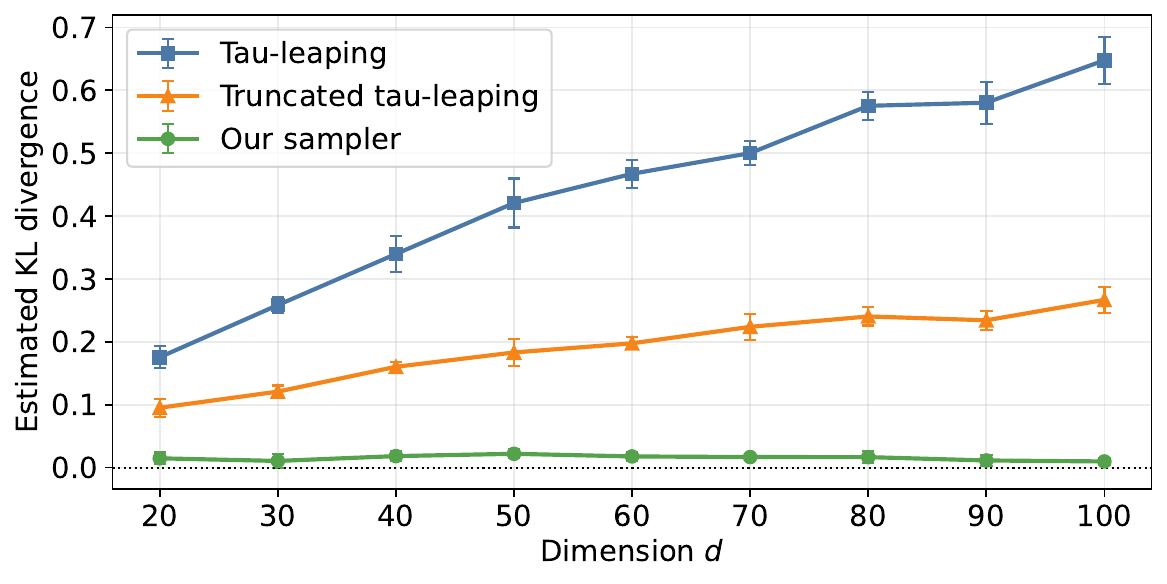}
\end{center}
\caption{\(N = 40\) discretization steps. Target distribution is the binary Markov chain with \(d p_{\mathrm{flip}} = 2\). We use early stopping at \(\delta = 1\mathrm{e}{-5}\) and \(T = 8\). Results are averaged over \(7\) runs. We compare three different samplers for the uniform process: \(\tau\)-leaping, truncated \(\tau\)-leaping, and ours.} 
\label{fig:compare-tau-leaping}
\end{figure}

\begin{figure}[t]
\begin{center}
\includegraphics[width=0.7\linewidth]{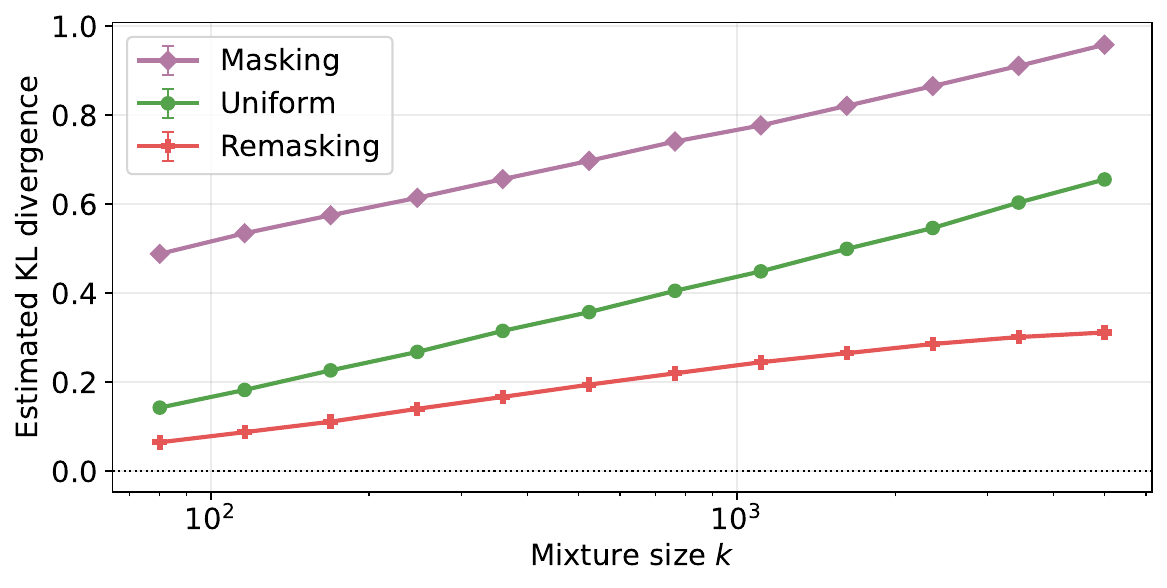}
\end{center}
\caption{\(N = 20\) discretization steps. Dimension \(d = 2000\). Target distribution is a sparse mixture with \(k\) components. We use early stopping at \(\delta = 1\mathrm{e}{-5}\) and \(T = 8\). Results are averaged over \(7\) runs. Masking uses constant grid; uniform and remasking use geometric grid.} 
\label{fig:mixture}
\end{figure}

In this section, we conduct experiments on two synthetic target distributions to illustrate the adaptive sampling behavior predicted by our theory. In particular, we examine how the sampling error depends on the ambient dimension $d$, the choice of time discretization, and the structural complexity of the target distribution.

For both target distributions considered, a binary Markov chain and a mixture of binary strings, the leave-one-out probabilities and thus the Bayes-optimal score function can be computed efficiently. We therefore use the exact values for our sampler, so that the approximation error vanishes and the remaining sampling error is solely due to discretization.

\paragraph{A binary Markov chain.}
We consider a Markov chain \(\{a_k\}_{k=1}^d\) on \(\{0, 1\}\) of length \(d\), where \(a_1 \sim \mathrm{Bern}(1/2)\) and for \(k = 1, \ldots, d - 1\), \(a_{k+1} = a_k\) with probability \(1 - p\), and \(a_{k+1} = 1 - a_k\) otherwise. We choose \(p = 2d^{-1}\), so that a typical sample consists of several (three on average) consecutive blocks of the same digit. Let \(q_{\mathrm{data}}\) be the distribution of this Markov chain. A simple computation shows 
\begin{equation*}
    \dtc(q_{\mathrm{data}}) \leq \ent(q_{\mathrm{data}}) = O(\log d).
\end{equation*}
\Cref{fig:mc-increase-d} compares three processes for the discrete diffusion models: (i) masking process, (ii) remasking process with \(p_M = \nicefrac{1}{2}\), and (iii) uniform process. While all three processes benignly depend on the dimension (given only \(N = 20\) of discretization steps), uniform and remasking processes consistently outperform the standard masking process. This can be attributed to the fact that both uniform and remasking processes allow the sampler to correct early mistakes, thus can perform well under extremely coarse discretization grids. 
We leave the theoretical justifications of this observation for future work.

\Cref{fig:compare-grid} shows two choices of discretization scheme for all three processes. We observe that the masking process performs much better on the constant grid, while uniform, in contrast, takes advantage of the geometric grid. Remasking also performs better on the geometric grid, but also performs well on the constant grid. 

\Cref{fig:compare-tau-leaping} compares three different samplers for the uniform process: \(\tau\)-leaping, its truncated version, where only one jump per coordinate is allowed per discretization step, and our sampler. We observe that our sampler shows the best performance out of the three samplers. The additional errors of the \(\tau\)-leaping and truncated \(\tau\)-leaping samplers arise from the non-zero approximation error of these samplers, as they inaccurately follow the Bayes-optimal sampler.

In these binary Markov chain experiments, we estimate the KL divergence by fitting an autoregressive model to the output samples. We use \(1500\) samples to fit the model parameters and \(2500\) samples for the KL estimation. Exact computations in small dimensions show good empirical agreement with this estimation. 

\paragraph{Mixture of binary strings.}
We independently sample $k$ binary strings \(s_1, \ldots, s_k \) uniformly from \(\{0, 1\}^d\) and define the empirical distribution
\begin{equation*}
    \widetilde q_{\mathrm{data}} = \frac{1}{k} \sum_{i=1}^k \delta_{s_i},
\end{equation*}
where \(\delta_x\) denotes the Dirac delta distribution at point \(x \in \{0, 1\}^d\). To ensure full support on \(\{0, 1\}^d\),
we consider the smoothed distribution
\begin{equation*}
    q_{\mathrm{data}}
    = (1-\varepsilon)\widetilde q_{\mathrm{data}}
      + \varepsilon \operatorname{Unif}(\{0,1\}^d),
    \qquad \varepsilon = 10^{-10}.
\end{equation*}
In the regime of \(\log k \ll d\), we have that
\begin{equation*}
    \dtc(q_{\mathrm{data}}) = O(\log k).
\end{equation*}
\Cref{fig:mixture} shows how the estimated KL divergence scales with increasing \(k\) from \(k = 80\) to \(k = 5000\). As the \(X\)-axis is plotted in logarithmic scale, we see that, for a fixed number of discretization steps \(N = 20\), the KL divergence grows logarithmically with \(k\), consistent with~\Cref{thm:main}. For this setting, we also observe that the uniform and remasking processes incur smaller sampling errors than the widely used masking process. Providing a rigorous explanation for this phenomenon is an interesting direction for future work.

In the mixture experiments, to estimate the KL divergence, we collect all the generated samples that do not match any of the \(k\) binary strings into a single bin and compute the KL divergence between this restricted distribution over \(k+1\) elements. While this only provides a lower bound on the true KL divergence, we find that this approximation is accurate in low dimensions and can scale to higher dimensions.
\section{Discussion}
This paper establishes adaptive sampling guarantees for the uniform and remasking processes. Our results show that the linear dependence on the ambient dimension $d$ exhibited by existing results for uniform discrete diffusion is not intrinsic to the forward process. Instead, the leave-one-out sampler studied here admits guarantees controlled by an information-theoretic measure of the target distribution. Thus, sampling efficiency depends not only on the choice of forward process, but also critically on how the reverse dynamics are approximated and discretized.
To the best of our knowledge, this is the first work to establish such adaptive guarantees for both the uniform and remasking processes. Our analysis also reveals a common structure underlying the two processes and suggests that the techniques developed here extend naturally to other discrete diffusion models used in practice. A key property required by our arguments is that the forward process is \emph{unstructured}: after a coordinate undergoes a jump, its new value is independent of its initial value.

Our work suggests several directions for future investigation. 

\begin{itemize}
    \item In our numerical experiments, the uniform and remasking processes consistently outperform the widely used masking process. Providing a theoretical explanation for this empirical observation is an interesting direction for future work. 
\item Extending the present techniques to other forward processes, such as discrete Gaussian or semantically dependent processes, may help clarify the connections between discrete and continuous diffusion models. 
\item It would also be interesting to determine whether higher-order samplers can further improve the dependence on accuracy or intrinsic complexity. 
\item Finally, characterizing the sample complexity required for accurate score estimation would be an important step toward a more unified theory of sampling with discrete diffusion models.
\end{itemize}

\section*{Acknowledgements}
This work is supported in part by Wharton Dean's Research Fund, the NSF grants CCF-2106778, CCF-2418156 and CAREER award DMS-2143215.  
This work is also supported by the NSF under Cooperative Agreement No. 2433450.

\appendix
\crefalias{section}{appendix}
\crefalias{subsection}{appendix}

\section{Details on the remasking process}
In this section, we provide explicit expressions for the probabilities used in the construction of the remasking process. We emphasize that the exact form of these expressions is not used in the proofs (with the exception of \(\nu(t, c)\), for \(c \in [S] \cup \{\mathrm{REMASK}\}\), used in the proofs of~\Cref{lem:inflow-exact,lem:approx-to-se}), and are given here for completeness. 
\begin{lemma}
\label{lem:nu-remask}
Consider the remasking process with parameter \(p_M\) and let \(\rho = \sqrt{1 - p_M}\). Recall \(\nu(t, c)\) from~\Cref{eq:def-nu}. We have
    \begin{equation*}
        \nu(t, b) = \frac{1}{1 - e^{-t}} \times \begin{cases} \frac{e^{-t}}{S}(\cosh(t \rho) - 1), \quad &\text{for } b \in [S], \\
        e^{-t} \rho \sinh(t\rho), \quad &\text{for } b = \mathrm{REMASK}, \\
        1 - e^{-t}\left( \cosh(t \rho) + \rho \sinh(t \rho)\right) , \quad &\text{for } b = \mask.
        \end{cases}
    \end{equation*}
\end{lemma}

\begin{proof}
    We assume \(0 < p_M < 1\) with the case \(p_M = 1\) interpreted by continuity from \(p_M \to 1\).
    To compute \(\nu(t, b)\), observe that it is enough to study a CTMC where all \([S]\) states are represented by a single state, and the rate matrix is as follows:
    \begin{equation*}
\cR =
\begin{pmatrix}
-1 & 1-p_M & p_M\\
1 & -1 & 0\\
0 & 0 & 0
\end{pmatrix}.
    \end{equation*}
    The order of the three states is \([S], \mathrm{REMASK}, \mask\). 
    Let \(B = {\footnotesize \begin{pmatrix}
-1 & 1-p_M \\
1 & -1 \\
\end{pmatrix}}\) be the top-left submatrix of \(\cR\) and observe that the corresponding part of \(\cR^k\) equals \(B^k\). Next, we have 
\begin{equation}
\label{eq:matrix-exp}
\begin{aligned}
    e^t\exp(t B) &= \sum_{k=0}^{\infty} \frac{t^k}{k!}\begin{pmatrix}
0 & 1-p_M \\
1 & 0 \\
\end{pmatrix}^k  \\
&= \sum_{m=0}^{\infty} \frac{t^{2m}}{(2m)!} (1 - p_M)^m I + \sum_{m=0}^{\infty} \frac{t^{2m + 1}}{(2m+1)!}(1-p_M)^m  \begin{pmatrix}
0 & 1-p_M \\
1 & 0 \\
\end{pmatrix} \\
&= \cosh(t\sqrt{1 - p_M}) I + \frac{1}{\sqrt{1 - p_M}} \sinh(t\sqrt{1 - p_M}) \begin{pmatrix}
0 & 1-p_M \\
1 & 0 \\
\end{pmatrix},
\end{aligned}
\end{equation}
and we compute for \(b \in [S]\) (recall that \(\rho = \sqrt{1 - p_M}\)),
\begin{equation*}
\begin{aligned}
        \nu(t, b) &= \frac{1}{S} \Pr(X_t^i \in [S] \mid \text{ jump on }[0, t]) \\
        &= \frac{1}{(1 - e^{-t}) S} \left(\Pr(X_t^i \in [S]) - \Pr(X_t^i \in [S] \text{ and no jump on }[0, t]\right) \\
        & = \frac{1}{(1 - e^{-t})S} \left(\left(\exp\left(t B \right)\right)_{11} - e^{-t}\right)\\
        &= \frac{e^{-t}}{(1-e^{-t})S} \left(\cosh(t \rho) - 1 \right).
\end{aligned}
    \end{equation*}
    Similarly,
    \begin{equation*}
        \begin{aligned}
        \nu(t, \mathrm{REMASK}) &=  \Pr(X_t^i = \mathrm{REMASK} \mid \text{ jump on }[0, t]) = \frac{\exp(tB)_{12}}{1 - e^{-t}}  = \frac{e^{-t} \rho \sinh(t \rho)}{1 - e^{-t}}.
\end{aligned}
    \end{equation*}
    Finally, \(\nu(t, \mask)\) follows from 
    \begin{equation*}
        \nu(t, \mask) + \nu(t, \mathrm{REMASK}) + \sum_{b \in [S]} \nu(t, b) = 1.
    \end{equation*}
\end{proof}

\begin{lemma}
\label{lem:cond-pr-remask}
Consider the remasking process with parameter \(p_M\) and let \(\rho = \sqrt{1 - p_M}\).
    Let \(0 \leq u < \ell\) with \(\Delta = \ell - u > 0\) and \(a, b \in \cV\). Then,
    \begin{equation*}
        \Pr(X_{\ell}^i = b \mid X_u^i = a) = \begin{cases}
             e^{-\Delta} \bI\{a = b\} + \frac{e^{-\Delta}}{S} (\cosh(\Delta \rho) - 1)&\text{if } a,b \in [S] \\
             e^{-\Delta} \rho \sinh(\Delta \rho)&\text{if } a \in [S], b = \mathrm{REMASK} \\
             1 - e^{-\Delta}(\cosh(\Delta \rho) + \rho \sinh(\Delta \rho))&\text{if } a \in [S], b = \mask, \\
             \frac{e^{-\Delta}}{S} \frac{\sinh(\Delta\rho)}{\rho} &\text{if } a = \mathrm{REMASK}, b \in [S], \\
            e^{-\Delta} \cosh(\Delta \rho) &\text{if } a = b = \mathrm{REMASK}, \\
             1 - e^{-\Delta}\left(\cosh(\Delta \rho) + \frac{\sinh(\Delta \rho)}{\rho}\right)  &\text{if } a = \mathrm{REMASK}, b = \mask \\
             1 &\text{if } a = b = \mask \\
             0 &\text{otherwise.} \\
        \end{cases}
    \end{equation*}
\end{lemma}
\begin{proof}
    The proof follows from~\Cref{eq:matrix-exp}.
\end{proof}
\begin{lemma}
    \label{lem:sampling-alg-exact-remask}
    Consider the remasking process with parameter \(p_M\) and let \(\rho = \sqrt{1 - p_M}\).
    Recall \(\wh \mu_i(b)\) defined in~\Cref{eq:sampling-alg}:
    \begin{equation*}
        \wh \mu_i(b) \ \propto\ \Pr(X_{\ell}^i = x_{\ell}^i \mid X_u^i = b)\left(\wh \Pr(X_0^i = b \mid X_{\ell}^{-i} = x_{\ell}^{-i}) e^{-u} + \nu(u, b)(1 - e^{-u})\right).
    \end{equation*}
    Then, for \(x_{\ell}^i \in [S]\),
    \begin{equation*}
        \wh \mu_i(b) \ \propto\ \begin{cases}
            \left(\wh \Pr(X_0^i = b \mid X_{\ell}^{-i} = x_{\ell}^{-i}) + \frac{
            \cosh(u \rho) - 1
            }{S}\right) \left(\bI\{b = x_{\ell}^i\} + \frac{\cosh(\Delta\rho) - 1}{S}\right), \quad &\text{if } b \in [S], \\
            \frac{1}{S} \sinh(u\rho)\sinh(\Delta\rho), \quad &\text{if } b = \mathrm{REMASK}, \\
            0 ,&\text{if } b = \mask.
        \end{cases}
    \end{equation*}
    For \(x_{\ell}^i = \mathrm{REMASK}\),
    \begin{equation*}
        \wh \mu_i(b) \ \propto\ \begin{cases}
                        \left(\wh \Pr(X_0^i = b \mid X_{\ell}^{-i} = x_{\ell}^{-i}) + \frac{
            \cosh(u \rho) - 1
            }{S}\right) \rho \sinh(\Delta \rho), \quad &\text{if } b \in [S], \\
            \rho\sinh(u\rho)\cosh(\Delta\rho), \quad &\text{if } b = \mathrm{REMASK}, \\
            0 ,&\text{if } b = \mask.
        \end{cases}
    \end{equation*}
    For \(x_{\ell}^i = \mask\),
    \begin{equation*}
        \wh \mu_i(b) \ \propto\ \begin{cases}
                        e^{-u}\left(\wh \Pr(X_0^i = b \mid X_{\ell}^{-i} = x_{\ell}^{-i}) + \frac{
            \cosh(u \rho) - 1
            }{S}\right) \left(1 - e^{-\Delta}\left(\cosh(\Delta \rho) + \rho \sinh(\Delta\rho)\right)\right), \quad &\text{if } b \in [S], \\
            e^{-u}\rho\sinh(u\rho)\left(1 - e^{-\Delta} \left(\cosh(\Delta \rho) + \frac{\sinh(\Delta \rho)}{\rho}\right)\right), \quad &\text{if } b = \mathrm{REMASK}, \\
            1 - e^{-u}\left(\cosh(u \rho) + \rho \sinh(u \rho)\right) ,&\text{if } b = \mask.
        \end{cases}
    \end{equation*}
    Normalization factor should be computed separately for each of the three cases: \(x_{\ell}^i \in [S], x_{\ell}^i = \mathrm{REMASK},\) and \(x_{\ell}^i = \mask\).
\end{lemma}

\begin{proof}
    The proof follows from~\Cref{lem:nu-remask,lem:cond-pr-remask}.
\end{proof}
\section{Technical preparations}
\label{sec:technical}
This section contains results that are used in the proofs. Importantly, all results here concern only the forward process, and not the reverse process. Recall the definitions of \emph{the total correlation} and \emph{the dual total correlation}: for a random vector \(X = (X^1, \ldots, X^d)\),
\begin{equation*}
\tc(X) \defn \sum_{i \in [d]} \ent(X^i) - \ent(X) \quad \text{and} \quad \dtc(X) \defn \ent(X) - \sum_{i \in [d]} \ent(X^i \mid X^{-i}).    
\end{equation*}
The following proposition is the basis of our main results, as it provides explicit expressions for the first and second partial derivatives of the mutual information.
\begin{proposition}
\label{prop:it-results}
Let $(X_t)_{t\in[0,T]}$ be a CTMC with rate matrices $(Q_t)_{t\in[0,T]}$ as in~\Cref{def:ctmc}. Then,
\begin{enumerate}[label=(\roman*)]
    \item 
    \[ \dtc(X_t) + \tc(X_t) = \sum_{i \in [d]} \mi(X_t^i\ ;\, X_t^{-i}); \]

    \item for $i \in [d]$, it satisfies 
    \[ \frac{\partial}{\partial v}\mi(X_t^i\ ;\,X_v^{-i}) = \bE_{x_t^i, x_v^{-i}} \left[\left(Q^{-i}_v \log \Pr(X_t^i = x_t^i \mid \cdot\ )\right)(x_v^{-i})\right]; \]

    \item for $i \in [d]$, we have
    \[ \frac{\partial}{\partial t}\mi(X_t^i\ ;\,X_v^{-i}) = \frac{\partial}{\partial t} \ent(X_t^i) + \bE_{x_t^i, x_v^{-i}} \left[\left(Q^{i}_t \log \Pr(\,\cdot \mid X_v^{-i} = x_v^{-i}\,)\right)(x_t^i)\right]; \]

    \item 
    \begin{align*}
        \sum_{i \in [d]} \frac{\partial}{\partial v} \mi(X_t^i\ ;\, X_v^{-i}) \bigg\rvert_{t = v} &= \frac{\d}{\d v} \dtc(X_v), \qquad \text{and} \qquad
        \sum_{i \in [d]} \frac{\partial}{\partial t} \mi(X_t^i\ ;\, X_v^{-i}) \bigg\rvert_{v = t} = \frac{\d}{\d t} \tc(X_t);
    \end{align*}

    \item for fixed $x_v^{-i}$, $j$, and $b$, let us define  
    \[ r(a) = \frac{\Pr(X_t^i = a \mid X_v^{-i} = x_v^{-i})}{\Pr(X_t^i = a \mid X_v^{-i} = x_v^{-i} \odot_j b)}.\]
    It then obeys
    \begin{align*}
        \frac{\partial^2}{\partial t \partial v} \mi(X_t^i\ ;\, X_v^{-i}) 
        &= \bE_{x_v^{-i}} \sum_{j \neq i} \sum_{a, b, c}\qtok_v(x_v^j, b) \qtok_t(a, c)\Pr(X_t^i = a \mid x_v^{-i} \odot_j b) \times \left[r(a) \log \frac{r(a)}{r(c)} + r(c) - r(a) \right].
    \end{align*}
\end{enumerate}
\end{proposition}

We use the following lemma to upper bound Bregman (Itakura-Saito) divergence by the KL divergence.
\begin{lemma}
\label{lem:ratio-ub}
    Let \(p, q \in [\alpha, 1]\) for some \(\alpha > 0\). Then,
    \begin{equation*}
        \frac{\frac p q  - 1 - \log \frac p q}{q\left(\frac p q \log \frac p q - \frac p q + 1 \right)}  \leq \frac 1 {\min(p, q)} \leq \frac 1 \alpha.   
    \end{equation*}
    The case \(p = q\) is interpreted by continuity.
\end{lemma}

The next definition introduces \(\cF(t, b)\), which informally quantifies the probability mass that is moved to the state \(b\) at time \(t\), given that a jump to the state \(b\) appeared. We recall that, in the uniform process case, this includes a possible self-loop jump \(b \to b\).
\begin{definition}
    \label{eq:def-ftb}
    Let \(t \geq 0\) and \(b \in \cV\). Define \(\cF: \bR_{\geq 0} \times \cV \to [0, 1]\) as follows:
    \begin{equation*}
    \inflow(t, b) \defn \frac{\d}{\d t} \Pr(X_t^i = b) + \Pr(X_t^i = b) = \sum_{a \in \cV} \qtok(a, b) \Pr(X_t^i = a).
    \end{equation*}
\end{definition}

\begin{lemma}
\label{lem:inflow-exact}
Let \(0 < t \leq T\) and \(b \in [S]\).
    Recall \(\nu(t, b) \defn \Pr(X_t^i = b \mid \exists \text{ a jump at the } i\text{-th coordinate on }[0, t]).\) Then,
    \begin{enumerate}
        \item for the uniform process, \(\inflow(t, b) = \nicefrac{1}{S}\),
        \item for the remasking process, \begin{equation*}
        \inflow(t, b) = \frac{e^{-t}}{S} \sqrt{1 - p_M} \sinh(t\sqrt{1 - p_M} ).
    \end{equation*}
    \end{enumerate}
    Consequently, for both processes, \(\nu(t, b) \gtrsim \inflow(t, b)\).
\end{lemma}
\begin{proof}
    For the uniform process, the result is immediate, as \(\qtok(a, b) = \nicefrac{1}{S}\) for all \(a, b \in [S]\) which implies \(\cF(t, b) = \nicefrac{1}{S}\) for all \(b \in [S]\). 
    For the remasking process, for \(b \in [S]\), the only non-zero element is \(\qtok(\mathrm{REMASK}, b) = \nicefrac{1}{S}\), which gives
    \begin{equation}
        \inflow(t, b) = \frac{\Pr(X_t^i = \mathrm{REMASK})}{S}.
    \end{equation}
Using~\Cref{lem:nu-remask}, we compute
\begin{equation*}
    \Pr(X_t^i = \mathrm{REMASK}) = \nu(t, \mathrm{REMASK})(1 - e^{-t}) = e^{-t} \sqrt{1 - p_M} \sinh(t \sqrt{1 - p_M}),
\end{equation*}
    and thus
    \begin{equation*}
        \frac{\nu(t, b)}{\inflow(t, b)} = \frac{\cosh(t \sqrt{1 - p_M}) - 1}{\sqrt{1 - p_M}(1 - e^{-t})\sinh(t \sqrt{1 - p_M})} =: f(t).
    \end{equation*}
Observe that \(f(t)\) is a strictly increasing function and \(\lim_{t \to 0} f(t) = 1/2\).
This shows that \(\nu(t, b) / \inflow(t, b) \geq 1/2\) and concludes the proof.
\end{proof}

Importantly, for both uniform and remasking processes, \(\cF(t, b)\) quantifies not only the marginal probability mass (averaged over the initial state \(X_0^i\)) moved to the state \(b\), but also when conditioning on a specific value \(X_0^i = c\) or on the context \(X_u^{-i} = x^{-i}\) at any time \(u\), as the following lemma shows.

\begin{lemma}
\label{lem:sum-prob}
    Consider the uniform or remasking process. For \(i \in [d]\), \(b \in \cV\), \(c\) in the support of \(X_0^i\), and \(t \geq 0\), we have
    \begin{equation}
    \label{eq:indep-initial}
        \sum_{a\in\cV} \qtok(a, b) \Pr(X_t^i = a \mid X_0^i = c) = \inflow(t, b).
    \end{equation}
    Consequently, for  \(i \in [d]\), \(b \in \cV\), \(x \in \cV^d\), and \(t, u \geq 0\), we have
    \begin{equation}
    \label{eq:indep-second}
        \sum_{a\in\cV} \qtok(a, b) \Pr(X_t^i = a \mid X_{u}^{-i} = x^{-i}) = \inflow(t, b),   
    \end{equation} 
    as long as \(\Pr(X_u^{-i} = x^{-i}) > 0\). Furthermore, the same holds for \(\wh \Pr(\cdot):\)
    \begin{equation*}
        \inflow(t, b) = \sum_{a\in\cV} \qtok(a, b) \wh \Pr(X_t^i = a \mid X_0^i = c) = \sum_{a\in\cV} \qtok(a, b) \wh \Pr(X_t^i = a \mid X_{u}^{-i} = x^{-i}).
    \end{equation*}
\end{lemma}
\section{Proof of our main results}
\label{sec:main-proofs}

\subsection{Proof of~\Cref{thm:main}}
    \Cref{prop:sampler} shows that a sample obtained from~\Cref{alg:sampler} has the same distribution as the one obtained using a CTMC \(\wh Q_t\) defined by~\Cref{eq:our-shat}.  In the following we analyze this CTMC. Using~\Cref{thm:main-err-decomp}, we have
    \begin{equation}
    \label{eq:main-thm-proof1}
        \KL(q_{T - t_N} \| p_{\mathrm{output}}) \leq \KL(q_T \| q_{\mathrm{noise}}) + \sum_{k=0}^{N -1}\cL_{\mathrm{approx}}^{(k)} + \sum_{k=0}^{N - 1}\cL_{\mathrm{discr}}^{(k)}.
    \end{equation}
    In view of \Cref{prop:main-discr}, the discretization error satisfies 
    \begin{equation}
    \label{eq:main-thm-proof2}
    \begin{aligned}
        \sum_{k=0}^{N - 1}\cL_{\mathrm{discr}}^{(k)} &\lesssim \sum_{k=0}^{N - 1} \frac{t_{k+1} - t_k}{\min(1, T - t_{k+1})} (\dtc(X_{T - t_{k+1}}) - \dtc(X_{T - t_k})) \\
        & \leq \kappa \sum_{k = 0}^{N - 1} (\dtc(X_{T - t_{k+1}}) - \dtc(X_{T - t_k})) \\
        & \leq \kappa \dtc(X_{T - t_{N}}) \\
        & \leq \kappa \dtc(X_0).
    \end{aligned}
    \end{equation}
    Next,~\Cref{lem:approx-to-se} gives
    \begin{equation}
    \label{eq:main-thm-proof3}
        \sum_{k=0}^{N - 1} \cL_{\mathrm{approx}}^{(k)} \leq \sum_{k=0}^{N - 1} (t_{k+1} - t_k) \cL_{\mathrm{SE}}(T - t_{k}, \wh s_{T - t_{k}}, s_{T - t_{k}}) \leq \escore,
    \end{equation}
    where the last inequality follows from~\Cref{asm:se-loss}.
    Collecting~\Cref{eq:main-thm-proof1,eq:main-thm-proof2,eq:main-thm-proof3} concludes the proof of~\Cref{eq:thm-kl-decomp}.

    Next, observe that under our condition on the step size, we can pick \(\kappa = O\left(\frac{T + \log \delta^{-1}}{N}\right),\) where \(\delta = T - t_N\) is the early stopping parameter. Let \(q_{\mathrm{noise}} = \mu^{\otimes d}\), where:
    \begin{equation*}
        \mu(b) = \bI\{b \in [S]\} \frac{e^{-T}}{S} + \nu(T, b) (1 - e^{-T}).
    \end{equation*}
    We have 
    \begin{equation}
    \label{eq:kl_qT_noise_1}
        \KL(q_T \| q_{\mathrm{noise}}) = \KL\left(q_T \bigg\lVert \bigotimes_{i\in[d]} q_T^i\right) + \sum_{i \in [d]} \KL(q_T^i \| \mu),
    \end{equation}
    where \(q_T^i\) is the \(i\)-th marginal of \(q_T\). Let \(X, Y \sim q_{0,T}\). By the convexity of the KL divergence,
    \begin{equation}
    \label{eq:kl_qT_noise_2}
        \KL(q_T^i \| \mu) \leq e^{-T} \KL\left(q_0^i \| \mathrm{Unif}([S])\right) = e^{-T}\left(\log S - \ent(X_i)\right).
    \end{equation}
     We also have
    \begin{equation}
    \label{eq:kl_qT_noise_3}
    \begin{aligned}
        \KL\left(q_T \bigg\lVert \bigotimes_{i\in[d]} q_T^i\right) &= \tc(Y) \\
        & = \sum_{i \in [d]} \ent(Y_i) - \ent(Y) \\
        & \leq \sum_{i \in [d]} \ent(Y_i) - \ent(Y \mid X) \\
        & = \sum_{i \in [d]} \left( \ent(Y_i) - \ent(Y_i \mid X_i)\right),
    \end{aligned}
    \end{equation}
 where the last line follows as the coordinates of \(Y\) are conditionally independent given \(X\). For fixed \(i \in [d]\),
 \begin{equation}
 \label{eq:kl_qT_noise_4}
     \ent(Y_i) - \ent(Y_i \mid X_i) = \mi(X_i\, ; Y_i) \leq \mi(X_i\, ; Y_i, B_i),
 \end{equation}
 for \(B_i = \bI\{\text{jump at the } i\text{-th coordinate on [0, T]}\}\) with \(B_i \sim \mathrm{Bern}(1 - e^{-T})\). As \(B_i \indep X_i, \) we obtain
 \begin{equation}
 \label{eq:kl_qT_noise_5}
     \mi(X_i\, ; Y_i, B_i) = \mi(X_i\, ; Y_i \mid B_i) = e^{-T} \ent(X_i).
 \end{equation}
 Collecting~Eqns.~(\ref{eq:kl_qT_noise_1}) to~(\ref{eq:kl_qT_noise_5})~gives
 \begin{equation*}
     \KL(q_T \| q_{\mathrm{noise}}) \leq e^{-T} \sum_{i \in [d]} \ent(X_i) + e^{-T} \sum_{i \in [d]}\left(\log S - \ent(X_i)\right) = e^{-T} d \log S.
 \end{equation*}
    Choosing \(T = O \left(\log (\varepsilon^{-1}d \log S)\right)\) concludes the proof.

\subsection{Proof of~\Cref{prop:sampler}}
\label{sec:pf-prof-sampler}

Fix a discretization interval $[u,\ell]$ and condition on the current
state $x_\ell$. By Eqn.~\eqref{eq:our-shat}, for a transition that changes
only coordinate $i$, the rate
\[
    \widehat Q_t(x,x\odot_i b)
\]
depends on the frozen context $x_\ell^{-i}$, but not on the evolving
coordinates $x^{-i}$. Consequently, on the interval $[u,\ell]$,
the approximate CTMC decomposes into $d$ independent one-dimensional
CTMCs. It therefore suffices to verify that, for each $i\in[d]$,
the update in Algorithm~\ref{alg:sampler} coincides with the transition
of the corresponding one-dimensional CTMC.

Fix $i\in[d]$ and $x_\ell^{-i}\in V^{d-1}$. To make the argument
explicit, introduce an auxiliary one-dimensional forward process
$(Z_t^i)_{t\in[0,\ell]}$ with rate matrix $Q$ and initial distribution
\[
    \Pr(Z_0^i=b)
    =
    \widehat{\Pr}
    \bigl(X_0^i=b\mid X_\ell^{-i}=x_\ell^{-i}\bigr).
\]
By construction and Eqn.~\eqref{eq:whpt-remask}, its marginal at time
$t$ is
\[
    \Pr(Z_t^i=b)
    =
    \widehat{\Pr}
    \bigl(X_t^i=b\mid X_\ell^{-i}=x_\ell^{-i}\bigr).
\]
Hence, the score in Eqn.~\eqref{eq:our-shat} is precisely the score of
this auxiliary one-dimensional process:
\[
    \widehat s_t(x\odot_i b,x)
    =
    \frac{\Pr(Z_t^i=b)}
         {\Pr(Z_t^i=x^i)}.
\]
It follows from the standard time-reversal formula for CTMCs that the
restriction of $\widehat Q_t$ to coordinate $i$ is exactly the reverse
generator of $(Z_t^i)$.

We initialize this reverse process at the observed endpoint
$Z_\ell^i=x_\ell^i$. Therefore, its distribution at time $u$ is $\Pr(Z_u^i=b\mid Z_\ell^i=x_\ell^i).$
By Bayes' rule,
\[
\begin{aligned}
    \Pr(Z_u^i=b\mid Z_\ell^i=x_\ell^i)
    &~\propto~
    \Pr(Z_\ell^i=x_\ell^i\mid Z_u^i=b)
    \Pr(Z_u^i=b) \\
    &=
    \Pr(X_\ell^i=x_\ell^i\mid X_u^i=b)
    \left[
        \widehat{\Pr}
        \bigl(X_0^i=b\mid X_\ell^{-i}=x_\ell^{-i}\bigr)e^{-u}
        +\nu(u,b)(1-e^{-u})
    \right],
\end{aligned}
\]
where the last equality follows from
Eqn.~\eqref{eq:whpt-remask}. The right-hand side is exactly
$\mu_i(b)$ in Eqn.~\eqref{eq:sampling-alg}. Thus, the update of
coordinate $i$ in Algorithm~\ref{alg:sampler} has the same law as the
corresponding coordinate of the CTMC generated by $\widehat Q_t$.

Since the coordinate processes are independent on each discretization
interval conditional on $x_\ell$, Algorithm~\ref{alg:sampler}, which
samples all coordinates independently in parallel, has the same
transition kernel from $x_\ell$ to $x_u$ as the CTMC with rate matrices
$\widehat Q_t$. Applying this argument successively over all
discretization intervals proves the claim.

\subsection{Proof of~\Cref{thm:main-err-decomp}}
\label{sec:pf-decomposition}
For the divergence \(\breg(\alpha, \gamma) = \frac \alpha \gamma - 1 - \log \frac \alpha \gamma\), a straightforward calculation shows, for any \(\alpha, \beta, \gamma > 0\):
\begin{equation*}
    \gamma \breg(\alpha, \gamma) = \beta \breg(\alpha, \beta) + \gamma\breg(\beta, \gamma) + (\beta - \gamma) \log \frac \alpha \beta.
\end{equation*}
Fix \(x_t \in \cV^d\) and let \(y = x \odot_i b\)  for \(i \in [d]\), \(b \in \cV\), such that \(Q_t(y, x_t) > 0\). We pick \(
    \alpha = \wh Q_{T-t}(x_t, y)\), \(\beta = \wt Q_{T-t}(x_t, y)\), and \(\gamma = \back Q_{T-t}(x_t, y)\)
and note that both \(\wh Q_{T-t}(x_t, y)\) and \(\wt Q_{T-t}(x_t, y)\) are functions of \(x_{\ell}\) and \(x_t^i\) but not of \(x_{t}^{-i}\). The law of total expectation gives
\begin{equation*}
\begin{aligned}
    &\bE_{x_{\ell}, x_t} \left[ \left(\wt Q_{T-t}(x_t, y) - \back Q_{T-t}(x_t, y)\right) \log \frac {\wh Q_{T-t}(x_t, y)} {\wt Q_{T-t}(x_t, y)} \right] \\
    &\qquad = \bE_{x_{\ell}, x_t^i} \left[ \left( \wt Q_{T-t}(x_t, y) - \bE_{x_t^{-i}}\left[\back Q_{T-t}(x_t, y) \mid X_
    {\ell} = x_{\ell}, X_t^i = x_{t}^i\right] \right) \log \frac {\wh Q_{T-t}(x_t, y)} {\wt Q_{T-t}(x_t, y)} \right] = 0,
\end{aligned}
\end{equation*}
as by~\Cref{prop:cond-exp}, \(\wt Q_{T-t}(x_t, y) = \bE_{x_{t}^{-i}}\left[\back Q_{T-t}(x_t, y) \mid X_{\ell} = x^{\ell}, X_t^i = x_{t}^i\right]\)\footnote{As \(X_t^{-i} \indep X_{\ell}^i \mid X_t^i\) from the Markovian property, conditioning on \(X_{\ell} = x_{\ell}\) and \(X_t^i = x_t^{i}\) is equivalent to conditioning on \(X_{\ell}^{-i} = x_{\ell}^{-i}\) and \(X_t^i = x_t^i\).}.

Putting these together, we have obtained the following decomposition:
        \begin{equation*}
        \begin{aligned}
        &\bE_{x_{\ell}, x_t}\  \Bigl[\back Q_{T-t}(x_t, y)\breg\Bigl(\wh Q_{T-t}(x_t, y), \back Q_{T-t}(x_t, y)\Bigr)\Bigr] \\
        &\quad = \bE_{x_{\ell}, x_t}\  \Bigl[\wt Q_{T-t}(x_t, y) \breg \Bigl(\wh Q_{T-t}(x_t, y), \wt Q_{T-t}(x_t, y)\Bigr)\Bigr] +  \bE_{x_{\ell}, x_t}\Bigl[\back Q_{T-t}(x_t, y) \breg \Bigl(\wt Q_{T-t}(x_t, y), \back Q_{T-t}(x_t, y)\Bigr)\Bigr].
    \end{aligned}
    \end{equation*}
    Using Girsanov's change-of-measure theorem~(\cite{campbell2022continuous}), we arrive at 
\begin{equation*}
\begin{aligned}\KL(q_{T - t_N} \| p_{\mathrm{output}}) \leq \KL(q_T \| q_{\mathrm{noise}}) &+ \int_{t_N}^{T} \bE \sum_{y \neq x_t} \back Q_{T-t}(x_t, y)\breg\Bigl(\wh Q_{T-t}(x_t, y), \back Q_{T-t}(x_t, y)\Bigr) \d t \\
 = \KL(q_T \| q_{\mathrm{noise}}) &+ \sum_{k=0}^{N - 1} \int_{T - t_{k+1}}^{T - t_k} \bE\sum_{y \neq x_t} \  \wt Q_{T-t}(x_t, y) \breg \Bigl(\wh Q_{T-t}(x_t, y), \wt Q_{T-t}(x_t, y)\Bigr) \d t \\
 &+ \sum_{k=0}^{N - 1} \int_{T - t_{k+1}}^{T - t_k} \bE\sum_{y \neq x_t} \  \back Q_{T-t}(x_t, y) \breg \Bigl(\wt Q_{T-t}(x_t, y), \back Q_{T-t}(x_t, y)\Bigr) \d t \\
 = \KL(q_T \| q_{\mathrm{noise}}) &+ \sum_{k=0}^{N -1}\cL_{\mathrm{approx}}^{(k)} + \sum_{k=0}^{N - 1}\cL_{\mathrm{discr}}^{(k)},
\end{aligned}
    \end{equation*}
    which concludes the proof.

\subsection{Discretization error control}
\label{sec:discr-proof}

\subsubsection{Proof of~\Cref{thm:main-discr}} 
Fix \(t \in [u, \ell]\).
    Using~\Cref{prop:cond-exp}, we arrive at
    \begin{equation*}
        \begin{aligned}
            &\bE_{x_t, x_{\ell}}\sum_{y \neq x_t} \back Q_{T - t}(x_t, y) \breg \Bigl(\wt Q_{T - t}(x_t, y), \back Q_{T - t}(x_t, y)\Bigr) \\
            &\quad =\bE_{x_t, x_{\ell}}\sum_{y \neq x_t} \left[\wt Q_{T - t}(x_t, y) - \back Q_{T - t}(x_t, y) - \back Q_{T - t}(x_t, y) \log \left(\frac{\wt Q_{T - t}(x_t, y)}{\back Q_{T - t}(x_t, y)}\right)\right] \\
            &\quad = \bE_{x_t, x_{\ell}}\sum_{y \neq x_t} \back Q_{T - t}(x_t, y) \log \frac{\back Q_{T - t}(x_t, y)}{\wt Q_{T - t}(x_t, y)} \\
            &\quad  = \bE_{x_t, x_{\ell}}\sum_{y \neq x_t}  Q_t(y, x_t) s_t(y, x_t) \log \frac{s_t(y, x_t)}{\wt s_t(y, x_t)} \\
            &\quad = \bE_{x_t, x_{\ell}} \sum_{i \in [d]} \sum_{b \in \cV} \qtok_t(b, x_t^i) s_t(x_t \odot_i b, x_t) \log \frac{s_t(x_t \odot_i b, x_t)}{\wt s_t(x_t \odot_i b, x_t)}.
    \end{aligned}
    \end{equation*}
    In the last line, we add the term corresponding to \(b = x_t^i\), as both \(s_t(x_t, x_t)\) and \(\wt s_t(x_t, x_t)\) equal to 1, and therefore this term equals to 0.
    Furthermore, note that the overall expression does not depend on \(x_{\ell}^i\), as 
    \begin{equation*}
        \wt s_t(x_t \odot_i b, x_t) = \frac{\Pr(X_t^i =  b \mid X_{\ell}^{-i} = x_{\ell}^{-i})}{\Pr(X_t^i =  x_t^i \mid X_{\ell}^{-i} = x_{\ell}^{-i})}.
    \end{equation*}
    We obtain
    \begin{equation*}
        \bE_{x_t, x_{\ell}} \sum_{i \in [d]} \sum_{b \in \cV} \qtok_t(b, x_t^i) s_t(x_t \odot_i b, x_t) \log \frac{s_t(x_t \odot_i b, x_t)}{\wt s_t(x_t \odot_i b, x_t)} = \sum_{i \in [d]} \sum_{b\in\cV} \bE_{x_t, x_{\ell}^{-i}}  \qtok_t(b, x_t^i) \frac{\Pr(x_t \odot_i b)}{\Pr(x_t)} \log \frac{s_t(x_t \odot_i b, x_t)}{\wt s_t(x_t \odot_i b, x_t)} ,
    \end{equation*}
    which, after relabeling \(x_t^i \leftrightarrow b\) and using that \(X_{\ell}^{-i} \indep X_t^i \mid X_t^{-i}\), gives
    \begin{equation*}
        \sum_{i \in [d]} \sum_{b\in\cV} \bE_{x_t, x_{\ell}^{-i}}  \qtok_t(b, x_t^i) \frac{\Pr(x_t \odot_i b)}{\Pr(x_t)} \log \frac{s_t(x_t \odot_i b, x_t)}{\wt s_t(x_t \odot_i b, x_t)} = \sum_{i \in [d]} \sum_{b\in\cV} \bE_{x_t, x_{\ell}^{-i}}  \qtok_t(x_t^i, b) \log \frac{s_t(x_t, x_t \odot_i b)}{\wt s_t(x_t, x_t \odot_i b)}.
    \end{equation*}
    We continue as follows:
    \begin{equation*}
        \begin{aligned}
             &\bE_{x_t, x_{\ell}^{-i}} \sum_{i \in [d]} \sum_{b \in \cV} \qtok_t(x_t^i, b)  \log \frac{s_t(x_t, x_t \odot_i b)}{\wt s_t(x_t, x_t \odot_i b)} \\
            & \quad = \bE_{x_t, x_{\ell}^{-i}} \sum_{i \in [d]} \sum_{b \in \cV} \qtok_t(x_t^i, b)  \left[\log \frac{\Pr(X_t^i = x_t^i \mid X_t^{-i} = x_t^{-i})}{\Pr(X_t^i = b \mid  X_t^{-i} = x_t^{-i})} - \log \frac{\Pr(X_t^i = x_t^i \mid X_\ell^{-i} = x_\ell^{-i})}{\Pr(X_t^i = b \mid  X_\ell^{-i} = x_\ell^{-i})}\right] \\
            & \quad = \sum_{i \in [d]}\left(\left(\frac \partial {\partial t} \mi(X_t^i\ ;\,X_v^{-i})\right)\biggr\rvert_{v={\ell}} - \left(\frac \partial {\partial t} \mi(X_t^i\ ;\,X_v^{-i})\right)\biggr\rvert_{v=t}\right) \\
            & \quad = \sum_{i \in [d]} \int_t^{\ell} \frac{\partial^2}{\partial v \partial t} \mi(X_t^i\ ;\,X_v^{-i}) \d v,
        \end{aligned}
    \end{equation*}
    where we used~\Cref{prop:it-results} (iii) in the third line. This proves~\Cref{eq:discr-err-exact}, as
    \begin{equation*}
        \cL_{\mathrm{discr}}^{(k)} = \int_{u}^{\ell} \bE_{x_{\ell}, x_t} \sum_{y \neq x_t} \back Q_{T-t}(x_t, y) \breg \Bigl(\wt Q_{T-t}(x_t, y), \back Q_{T-t}(x_t, y)\Bigr) \d t = \sum_{i\in[d]}\int_{u}^{\ell} \int_t^{\ell} \frac{\partial^2}{\partial t \partial v}\mi(X_t^i \ ;\,X_v^{-i})\d v \d t.
    \end{equation*}

\subsubsection{Proof of \Cref{prop:main-discr}}
\Cref{thm:main-discr} together with~Fubini's theorem imply:

    \begin{equation}
    \label{eq:discr-fubini}
   \begin{aligned}
       \cL_{\mathrm{discr}}^{(k)} &=  \sum_{i\in[d]}\int_{u}^{\ell} \int_t^{\ell} \frac{\partial^2}{\partial t \partial v}\mi(X_t^i \ ;\,X_v^{-i})\d v \d t \\
       & = \sum_{i\in[d]}\int_{u}^{\ell} \int_u^{v} \frac{\partial^2}{\partial t \partial v}\mi(X_t^i \ ;\,X_v^{-i})\d t \d v \\ 
       & = \sum_{i\in[d]}\int_{u}^{\ell}\frac{\partial}{\partial v} \mi(X_t^i \ ;\, X_v^{-i})\biggr\rvert_{t=v} - \frac{\partial}{\partial v} \mi(X_t^i \ ;\, X_v^{-i})\biggr\rvert_{t=u} \d v. 
    \end{aligned}
   \end{equation}
   Next,~\Cref{cor:second-der-ub} gives
   \begin{equation*}
       -\frac{\partial}{\partial v} \mi(X_t^i \ ;\, X_v^{-i})\biggr\rvert_{t=u} = -\left(1 + O(\kappa)\right) \frac{\partial}{\partial v} \mi(X_t^i \ ;\, X_v^{-i}) \biggr\rvert_{t=v},
   \end{equation*}
   which implies using~\Cref{eq:discr-fubini},
   \begin{equation*}
       \cL_{\mathrm{discr}}^{(k)} \lesssim - \kappa \int_{u}^{\ell}\sum_{i\in[d]}\frac{\partial}{\partial v} \mi(X_t^i \ ;\, X_v^{-i})\biggr\rvert_{t=v} \d v.
   \end{equation*}
   From~\Cref{prop:it-results} (iv) we have \(\frac{\partial}{\partial v} \mi(X_t^i \ ;\, X_v^{-i})\Bigr\rvert_{t=v} = \frac{\d}{\d v} \dtc(X_v)\), therefore
   \begin{equation*}
    \begin{aligned}
        \cL_{\mathrm{discr}}^{(k)} \lesssim -\kappa \int_u^{\ell} \frac{\d}{\d v}\dtc(X_v) \d v = \kappa(\dtc(X_{u}) - \dtc(X_{\ell})).
    \end{aligned}
   \end{equation*}
   This concludes the proof.

\subsubsection{Statement and proof of \Cref{cor:second-der-ub}}

\begin{proposition}
\label{cor:second-der-ub}
Let \((X_t)_{t \in [0, T]}\) be the uniform or remasking process. Then, for \(i \in [d]\) and \(t, v > 0\),
\begin{equation}
\label{eq:bound-second-der-by-first}
    \frac{\partial^2}{\partial t \partial v} \mi(X_t^i\ ;\, X_v^{-i}) \lesssim -\left(1 + \frac{1}{1 - e^{-t}}\right)\frac{\partial}{\partial v} \mi(X_t^i\ ;\, X_v^{-i}).
\end{equation}
Furthermore, if \(v \geq u\) with \(v - u \leq \kappa \min(1, u)\),
 \begin{equation*}
    -\frac{\partial}{\partial v} \mi(X_t^i \ ;\, X_v^{-i})\biggr\rvert_{t=u} = -\left(1 + O(\kappa)\right) \frac{\partial}{\partial v} \mi(X_t^i \ ;\, X_v^{-i}) \biggr\rvert_{t=v}.
\end{equation*}
\end{proposition}

\begin{proof}
Recall~\Cref{prop:it-results} (ii) and (v):
\begin{equation*}
\begin{aligned}
 \frac{\partial}{\partial v} \mi(X_t^i\ ;\, X_v^{-i}) &= -\bE_{x_v^{-i}} \sum_{j \neq i} \sum_{a, b \in \cV}\qtok(x_v^j, b) \Pr(X_t^i = a \mid x_v^{-i} \odot_j b) \left[r(a) \log r(a) - r(a) +1 \right],
 \\
    \frac{\partial^2}{\partial t \partial v} \mi(X_t^i\ ;\, X_v^{-i}) &= \bE_{x_v^{-i}} \sum_{j \neq i} \sum_{a, b, c\in\cV}\qtok(x_v^j, b) \qtok(a, c)\Pr(X_t^i = a \mid x_v^{-i} \odot_j b) \left[r(a) \log \frac{r(a)}{r(c)} + r(c) - r(a) \right],
\end{aligned}
\end{equation*}
where, for fixed \(x_v^{-i}, j\), and \(b\), we recall \(r(a) = \Pr(X_t^i = a \mid X_v = x_v^{-i})\ /\,\Pr(X_t^i = a \mid X_v = x_v^{-i} \odot_j b)\). We  decompose
\begin{equation*}
    r(a) \log \frac{r(a)}{r(c)} + r(c) - r(a) = \Bigr(r(a) \log r(a) - r(a) + 1\Bigr) + \Bigr(r(c) - 1 - r(a) \log r(c)\Bigl),
\end{equation*}
which, together with~\Cref{lem:sum-prob}, immediately shows (relabeling \(a \leftrightarrow c\) in the second line)
\begin{equation*}
\begin{aligned}
     &\frac{\partial^2}{\partial t \partial v} \mi(X_t^i\ ;\, X_v^{-i}) + \frac{\partial}{\partial v} \mi(X_t^i\ ;\, X_v^{-i}) \\
     &\quad = \bE_{x_v^{-i}} \sum_{j \neq i} \sum_{a, b, c \in \cV}\qtok(x_v^j, b) \qtok(a, c)\Pr(X_t^i = a \mid x_v^{-i} \odot_j b) \left[r(c) - 1 - r(a) \log r(c) \right] \\
     & \quad = \bE_{x_v^{-i}} \sum_{j \neq i} \sum_{a, b\in \cV}\qtok(x_v^j, b)  \inflow(t, a)\left[r(a) - 1 - \log r(a) \right].
\end{aligned}
\end{equation*}
Next, we apply~\Cref{lem:ratio-ub,lem:inflow-exact}, and obtain
\begin{equation*}
\begin{aligned}
    \frac{\partial^2}{\partial t \partial v} \mi(X_t^i\ ;\, X_v^{-i}) &=-  \frac{\partial}{\partial v} \mi(X_t^i\ ;\, X_v^{-i}) + \bE_{x_v^{-i}} \sum_{j \neq i} \sum_{a, b\in\cV}\qtok(x_v^j, b) \inflow(t, a)\left[r(a) - 1 - \log r(a) \right] \\
    & \lesssim -  \frac{\partial}{\partial v} \mi(X_t^i\ ;\, X_v^{-i}) \\
    &\qquad \quad + \frac{1}{1 - e^{-t}}\bE_{x_v^{-i}} \sum_{j \neq i} \sum_{a, b\in\cV}\qtok(x_v^j, b) \Pr(X_t^i = a \mid x_v^{-i} \odot_j b) \left[r(a) \log r(a) - r(a) +1 \right] \\
    & = -\left(1 + \frac{1}{1 - e^{-t}}\right)\frac{\partial}{\partial v} \mi(X_t^i\ ;\, X_v^{-i}).
\end{aligned}
\end{equation*}
This finishes the proof of~\Cref{eq:bound-second-der-by-first}. Next, we use this bound in the following variation of Gr\"{o}nwall's inequality:
    \begin{equation*}
        \frac{\partial}{\partial v}\mi(X_t^i \ ;\, X_v^{-i}) \biggr\rvert_{t=v} \leq \exp\left(-C\int_u^v \left(1 + \frac{1}{1 - e^{-t}}\right) \d t \right)\frac{\partial}{\partial v} \mi(X_t^i \ ;\, X_v^{-i})\biggr\rvert_{t=u},
    \end{equation*}
    for some universal constant \(C > 0\).
    This gives
    \begin{equation*}
       -\frac{\partial}{\partial v} \mi(X_t^i \ ;\, X_v^{-i})\biggr\rvert_{t=u} \leq -e^{C(v - u)} \left(\frac{e^v - 1}{e^u - 1}\right)^C \frac{\partial}{\partial v} \mi(X_t^i \ ;\, X_v^{-i}) \biggr\rvert_{t=v}.
    \end{equation*}
Under the condition that \(v - u \leq  \kappa \min(1, u)\), we have that 
\begin{equation*}
    e^{C(v-u)}\cdot\left(\frac{e^v - 1}{e^u - 1}\right)^C = \left(e^{v-u}\left(e^{v-u} + \frac{e^{v-u} - 1}{e^u - 1}\right)\right)^C = (1 + O(\kappa)) \left(1 + O\left(\kappa + \frac{\kappa u e^\kappa}{u}\right)\right) = 1 + O(\kappa),
\end{equation*}
which leads to 
\begin{equation*}
    -\frac{\partial}{\partial v} \mi(X_t^i \ ;\, X_v^{-i})\biggr\rvert_{t=u} = -\left(1 + O(\kappa)\right) \frac{\partial}{\partial v} \mi(X_t^i \ ;\, X_v^{-i}) \biggr\rvert_{t=v}.
\end{equation*}
This concludes the proof.
\end{proof}

\subsection{Approximation error control}
\label{sec:approx-proof}

\subsubsection{Proof of~\Cref{thm:main-approx}}

To facilitate the proof of \Cref{thm:main-approx}, let us introduce the following proposition. 

\begin{proposition}
    \label{lem:approx-exact}
    Fix \(k \in \{0, \ldots, N - 1\}\) in the time discretization \(0 = t_0 < t_1 < \ldots < t_N < T\). Consider the uniform or remasking process and recall~\(\inflow(t, b)\) from~\Cref{eq:def-ftb}. Letting \(u = T - t_{k+1}\) and \(\ell = T - t_k\), we have

    \begin{equation*}
    \begin{aligned}
        &\cL_{\mathrm{approx}}^{(k)} = \sum_{i\in[d]} \sum_{b\in\cV}\int_u^{\ell} \bE_{x_{\ell}^{-i}} \Biggl[ \Pr(X_t^i = b \mid x_{\ell}^{-i}) \log \frac{\Pr(X_t^i = b \mid x_{\ell}^{-i})}{\wh \Pr(X_t^i = b \mid x_{\ell}^{-i})} \\
        &\qquad \qquad \qquad \qquad \qquad \qquad + \inflow(t, b)\left(\frac{\Pr(X_t^i = b \mid x_{\ell}^{-i})}{\wh \Pr(X_t^i = b \mid x_{\ell}^{-i})} - 1 - \log \frac{\Pr(X_t^i = b \mid x_{\ell}^{-i})}{\wh \Pr(X_t^i = b \mid x_{\ell}^{-i})} \right) \Biggr]\d t.
    \end{aligned}
    \end{equation*}
\end{proposition}

Proposition~\ref{lem:approx-exact} leads to the following control of the approximation error. The proofs of these two results are deferred to later in the section. 

\begin{corollary}
\label{prop:main-approx}
    Fix \(k \in \{0, \ldots, N - 1\}\) in the time discretization \(0 = t_0 < t_1 < \ldots < t_N < T\). Consider the uniform or remasking process and let \(u = T - t_{k+1}\), \(\ell = T - t_k\). We have
\begin{equation*}
\cL_{\mathrm{approx}}^{(k)} \lesssim \frac{\ell - u}{\min(1, u)} \sum_{i \in [d]}  e^{-u} \KL \left( \mu_{X_0^i} \ \big\|\ \wh\mu_{X_0^i} \ \Bigr|\  X_{\ell}^{-i}\,\right).
    \end{equation*}
\end{corollary}

In view of ~\Cref{prop:main-approx}, we arrive at 
    \begin{equation*}
    \begin{aligned}
        \sum_{k=0}^{N - 1} \cL_{\mathrm{approx}}^{(k)} &\lesssim \sum_{k=0}^{N - 1} \frac{t_{k+1} - t_k}{\min(1, T - t_{k+1})} \sum_{i \in [d]} e^{-(T - t_{k+1})} \KL\left(\mu_{X_0^i}\ \big\|\ \wh \mu_{X_0^i}\ \Bigr\vert\ X_{T - t_k}^{-i} \right) \\
        & \leq \kappa \sum_{k=0}^{N - 1}  \sum_{i \in [d]} e^{-(T - t_{k+1})} \KL \left( \mu_{X_0^i} \ \big\|\ \wh\mu_{X_0^i} \ \Bigr|\  X_{T - t_k}^{-i}\,\right),
    \end{aligned}
    \end{equation*}
which concludes the proof of ~\Cref{thm:main-approx}.

\begin{proof}[Proof of Proposition~\ref{lem:approx-exact}.]
    Recall that for \(t \in [u, \ell]\), we have
    \begin{equation*}
        \wt s_t(x \odot_i b, x) = \frac{\Pr(X_t^i = b \mid X_{\ell}^{-i} = x^{-i})}{\Pr(X_t^i = x^i \mid X_{\ell}^{-i} = x^{-i})} \qquad \text{and} \qquad 
        \wh s_t(x \odot_i b, x) = \frac{\wh \Pr(X_t^i = b \mid X_{\ell}^{-i} = x^{-i})}{\wh \Pr(X_t^i = x^i \mid X_{\ell}^{-i} = x^{-i})}.
    \end{equation*}
    Here and below, for clarity of the exposition, we omit stochastic process notation \(X_t^i = \cdot\,, X_v^i = \cdot\,,\) etc., when time and dimension indices are clear from the context. Thus, e.g., \(\Pr(x_{\ell}^i \mid x_{\ell}^{-i})\) stands for \(\Pr(X_\ell^i = x_{\ell}^i \mid X_{\ell}^{-i} = x_{\ell}^{-i})\).
    Using~\Cref{eq:def-approx}, we write
    \begin{equation}
    \label{eq:err-approx-first}
    \begin{aligned}
        \cL_{\mathrm{approx}}^{(k)} &= \int_u^{\ell} \bE_{x_t, x_{\ell}} \sum_{y \neq x_t} \wt Q_{T - t}(x_t, y) \breg\left(\wh Q_{T - t}(x_t, y), \wt Q_{T - t}(x_t, y)\right) \d t \\
        & = \int_u^{\ell} \bE_{x_t, x_{\ell}} \sum_{y \neq x_t} Q_t(y, x_t) \wt s_t(y, x_t) \breg\left(\wh s_t(y, x_t), \wt s_t(y, x_t)\right) \d t \\
        & = \int_u^{\ell} \sum_{i \in [d]} \bE_{x_{\ell}^{-i}} \sum_{x_{\ell}^i, x_t^i \in \cV} \Pr(x_{\ell}^{i} \mid x_{\ell}^{-i}) \Pr(x_t^i \mid x_{\ell})  \\
        & \qquad \qquad \qquad \quad \quad \times \sum_{b \in \cV} \qtok(b, x_t^i) \wt s_t(x_t \odot_i b, x_t) \breg\left(\wh s_t(x_t \odot_i b, x_t), \wt s_t(x_t \odot_i b, x_t)\right) \d t. \\     
    \end{aligned}
    \end{equation}
    The terms corresponding to \(b = x_t^i\) are added as their contribution is 0 by \(\breg(\wh s_t(x_t, x_t), \wt s_t(x_t, x_t)) = \breg(1, 1) = 0\). We also recall that we  define diagonal elements of \(\qtok\) such that \(\sum_{b \in \cV} \qtok(a, b) = 1\), for all \(a \in \cV\).
    Using the equality
    \begin{equation*}
        \frac{\Pr(x_t^i \mid x_{\ell})}{\Pr(x_t^i \mid x_{\ell}^{-i})} = \frac{\Pr(x_t^i, x_\ell)}{\Pr(x_t^i, x_\ell^{-i})} \cdot \frac{\Pr(x_\ell^{-i})}{\Pr(x_\ell)} = \frac{\Pr(x_\ell \mid x_t^i)}{\Pr(x_\ell^{-i} \mid x_t^i)} \cdot \frac{1}{\Pr(x_{\ell}^i \mid x_{\ell}^{-i})} = \frac{\Pr(x_{\ell}^i \mid x_t^i)}{\Pr(x_{\ell}^i \mid x_{\ell}^{-i})},
    \end{equation*}
    we can rewrite 
    \begin{equation*}
        \Pr(x_{\ell}^{i} \mid x_{\ell}^{-i}) \Pr(x_t^i \mid x_{\ell}) \wt s_t(x_t \odot_i b, x_t) = \Pr(x_{\ell}^{i} \mid x_{\ell}^{-i}) \Pr(x_t^i \mid x_{\ell}) \frac{\Pr(X_t^i = b \mid x_{\ell}^{-i})}{\Pr(x_t^i \mid x_{\ell}^{-i})} = \Pr(X_t^i = b \mid x_{\ell}^{-i}) \Pr(x_{\ell}^i \mid x_t^i).
    \end{equation*}
    Plugging this into~\Cref{eq:err-approx-first}, we continue:
    \begin{equation*}
    \begin{aligned}
    \cL_{\mathrm{approx}}^{(k)} &= \int_u^{\ell} \sum_{i \in [d]} \bE_{x_{\ell}^{-i}} \sum_{x_{\ell}^i, x_t^i, b \in \cV}  \qtok(b, x_t^i) \Pr(b \mid x_{\ell}^{-i}) \Pr(x_{\ell}^i \mid x_t^i) \breg\left(\wh s_t(x_t \odot_i b, x_t), \wt s_t(x_t \odot_i b, x_t)\right) \d t \\
    &= \int_u^{\ell} \sum_{i \in [d]} \bE_{x_{\ell}^{-i}}  \sum_{x_{t}^i, b \in\cV} \qtok(b, x_t^i) \Pr(b \mid x_{\ell}^{-i}) \breg\left(\wh s_t(x_t \odot_i b, x_t), \wt s_t(x_t \odot_i b, x_t)\right) \d t
    \end{aligned}
    \end{equation*}
    Next, we use:
    \begin{equation}
    \label{eq:breg-p-hatp}
        \breg\left(\wh s_t(x_t \odot_i b, x_t), \wt s_t(x_t \odot_i b, x_t)\right)  = \frac{\wh \Pr(X_t^i = b \mid x_{\ell}^{-i}) \Pr(x_t^i \mid x_{\ell}^{-i}) }{\wh \Pr(x_t^i \mid x_{\ell}^{-i}) \Pr(X_t^i = b \mid x_{\ell}^{-i})} - 1 - \log \frac{\wh \Pr(X_t^i = b \mid x_{\ell}^{-i})  }{\Pr(X_t^i = b \mid x_{\ell}^{-i}) } + \log \frac{\wh \Pr(x_t^i \mid x_{\ell}^{-i})}{\Pr(x_t^i \mid x_{\ell}^{-i})}.
    \end{equation}
    Splitting~\Cref{eq:breg-p-hatp} into four parts, we write \(\cL_{\mathrm{approx}}^{(k)} = \int_u^{\ell} \sum_{i \in [d]} \left[I_1 + I_2 + I_3 + I_4\right]\d t\), where, for fixed \(t \in [u, \ell]\) and \(i \in [d]\):
    \begin{align*}
        I_1 &=  \bE_{x_{\ell}^{-i}}  \sum_{x_{t}^i,b\in\cV}\qtok(b, x_t^i) \Pr(X_t^i = b \mid x_{\ell}^{-i}) \frac{\wh \Pr(X_t^i = b \mid x_{\ell}^{-i}) \Pr(x_t^i \mid x_{\ell}^{-i}) }{\wh \Pr(x_t^i \mid x_{\ell}^{-i}) \Pr(X_t^i = b \mid x_{\ell}^{-i})} \\
        & = \bE_{x_{\ell}^{-i}}  \sum_{x_{t}^i,b\in\cV} \qtok(b, x_t^i)  \frac{\wh \Pr(X_t^i = b \mid x_{\ell}^{-i}) \Pr(x_t^i \mid x_{\ell}^{-i}) }{\wh \Pr(x_t^i \mid x_{\ell}^{-i}) } \\
        & = \bE_{x_{\ell}^{-i}}  \sum_{x_{t}^i \in \cV} \frac{ \Pr(x_t^i \mid x_{\ell}^{-i}) }{\wh \Pr(x_t^i \mid x_{\ell}^{-i}) }  \left(\sum_{b\in\cV}   \qtok(b, x_t^i)  \wh \Pr(X_t^i = b \mid x_{\ell}^{-i}) \right),\\
            I_2 &= -\bE_{x_{\ell}^{-i}}  \sum_{x_{t}^i,b\in\cV} \qtok(b, x_t^i) \Pr(X_t^i = b \mid x_{\ell}^{-i}), \\
            I_3 &=\bE_{x_{\ell}^{-i}}   \sum_{b\in\cV}  \left( \sum_{x_t^i\in\cV} \qtok(b, x_t^i) \right) \Pr(X_t^i = b \mid x_{\ell}^{-i}) \log \frac{ \Pr(X_t^i = b \mid x_{\ell}^{-i})  }{\wh \Pr(X_t^i = b \mid x_{\ell}^{-i}) }, \qquad \text{and}\\
            I_4 & = -\bE_{x_{\ell}^{-i}}  \sum_{x_{t}^i\in\cV}  \log \frac{ \Pr(x_t^i \mid x_{\ell}^{-i})}{\wh \Pr(x_t^i \mid x_{\ell}^{-i})} \left(\sum_{b\in\cV}  \qtok(b, x_t^i) \Pr(X_t^i = b \mid x_{\ell}^{-i})\right).
        \end{align*}
    Collecting \(I_1 + I_2 + I_4\) gives:
    \begin{equation*}
\begin{aligned}
        I_1 + I_2 + I_4  &= \bE_{x_{\ell}^{-i}} \sum_{b\in\cV} \Biggl[ \left(\sum_{a\in\cV}  \qtok(a, b) \wh \Pr(X_t^i = a \mid x_{\ell}^{-i})\right) \left(\frac{ \Pr(X_t^i = b \mid x_{\ell}^{-i}) }{\wh \Pr(X_t^i = b \mid x_{\ell}^{-i}) }\right)\\
        &\qquad \qquad \qquad - \left(\sum_{a\in\cV}  \qtok(a, b)  \Pr(X_t^i = a \mid x_{\ell}^{-i})\right) \left(  1 + \log \frac{ \Pr(X_t^i = b \mid x_{\ell}^{-i}) }{\wh \Pr(X_t^i = b \mid x_{\ell}^{-i}) }\right)\Biggr] \\
        &  = 
        \bE_{x_{\ell}^{-i}} \sum_{b\in\cV} \inflow(t, b)\left(\frac{ \Pr(X_t^i = b \mid x_{\ell}^{-i}) }{\wh \Pr(X_t^i = b \mid x_{\ell}^{-i}) } -   1 - \log \frac{ \Pr(X_t^i = b \mid x_{\ell}^{-i}) }{\wh \Pr(X_t^i = b \mid x_{\ell}^{-i}) }\right),
    \end{aligned}
    \end{equation*}
    where we used~\Cref{lem:sum-prob} in the last line. This proves that
    \begin{equation*}
    \begin{aligned}
        &\cL_{\mathrm{approx}}^{(k)} = \int_u^{\ell} \sum_{i \in [d]} \Biggl[ I_3 + \bE_{x_{\ell}^{-i}} \sum_{b\in \cV} \inflow(t, b)\left(\frac{\Pr(X_t^i = b \mid x_{\ell}^{-i})}{\wh \Pr(X_t^i = b \mid x_{\ell}^{-i})} - 1 - \log \frac{\Pr(X_t^i = b \mid x_{\ell}^{-i})}{\wh \Pr(X_t^i = b \mid x_{\ell}^{-i})} \right) \Biggr]\d t.
    \end{aligned}
    \end{equation*}
    Observe that by our definition of the diagonal elements of \(\qtok\), we have \( \sum_{x_t^i} \qtok(b, x_t^i) = 1\), and therefore,
    \begin{equation*}
    \begin{aligned}
    I_3 &= \bE_{x_{\ell}^{-i}}   \sum_{b\in \cV}  \left( \sum_{x_t^i \in \cV} \qtok(b, x_t^i) \right) \Pr(X_t^i = b \mid x_{\ell}^{-i}) \log \frac{ \Pr(X_t^i = b \mid x_{\ell}^{-i})  }{\wh \Pr(X_t^i = b \mid x_{\ell}^{-i}) }\\
    &= \bE_{x_{\ell}^{-i}}   \sum_{b\in\cV}  \Pr(X_t^i = b \mid x_{\ell}^{-i}) \log \frac{ \Pr(X_t^i = b \mid x_{\ell}^{-i})  }{\wh \Pr(X_t^i = b \mid x_{\ell}^{-i}) }.
    \end{aligned}
    \end{equation*}
    We obtain
\begin{equation*}
    \begin{aligned}
        &\cL_{\mathrm{approx}}^{(k)} = \sum_{i\in[d]} \sum_{b\in\cV}\int_u^{\ell} \bE_{x_{\ell}^{-i}} \Biggl[ \Pr(X_t^i = b \mid x_{\ell}^{-i}) \log \frac{\Pr(X_t^i = b \mid x_{\ell}^{-i})}{\wh \Pr(X_t^i = b \mid x_{\ell}^{-i})} \\
        &\qquad \qquad \qquad \qquad \qquad \qquad + \inflow(t, b)\left(\frac{\Pr(X_t^i = b \mid x_{\ell}^{-i})}{\wh \Pr(X_t^i = b \mid x_{\ell}^{-i})} - 1 - \log \frac{\Pr(X_t^i = b \mid x_{\ell}^{-i})}{\wh \Pr(X_t^i = b \mid x_{\ell}^{-i})} \right) \Biggr]\d t,
    \end{aligned}
    \end{equation*}
which concludes the proof.
\end{proof}

\begin{proof}[Proof of \Cref{prop:main-approx}]
\Cref{lem:approx-exact} shows:
\begin{equation*}
    \begin{aligned}
        &\cL_{\mathrm{approx}}^{(k)} = \sum_{i\in[d]} \sum_{b\in\cV}\int_u^{\ell} \bE_{x_{\ell}^{-i}} \Biggl[ \Pr(X_t^i = b \mid x_{\ell}^{-i}) \log \frac{\Pr(X_t^i = b \mid x_{\ell}^{-i})}{\wh \Pr(X_t^i = b \mid x_{\ell}^{-i})} \\
        &\qquad \qquad \qquad \qquad \qquad \qquad + \inflow(t, b)\left(\frac{\Pr(X_t^i = b \mid x_{\ell}^{-i})}{\wh \Pr(X_t^i = b \mid x_{\ell}^{-i})} - 1 - \log \frac{\Pr(X_t^i = b \mid x_{\ell}^{-i})}{\wh \Pr(X_t^i = b \mid x_{\ell}^{-i})} \right) \Biggr]\d t.
    \end{aligned}
    \end{equation*}
    Without loss of generality, we assume \(b \in [S]\). Indeed, in the uniform process, \(\cV = [S]\), and in the remasking process, as the target distribution is not supported on \(\{\mask, \mathrm{REMASK}\}\), we have that \(\Pr(X_t^i = \mask \mid x_{\ell}^{-i}) = \nu(t, \mask) = \wh \Pr(X_t^i = \mask \mid x_{\ell}^{-i})\). The same holds for \(b = \mathrm{REMASK}\). Thus, terms in the sum corresponding to \(b \in \{\mask, \mathrm{REMASK}\}\) do not contribute to \(\cL_{\mathrm{approx}}^{(k)}\) and we may assume \(b \in [S]\).
    Using~\Cref{eq:pr-t-from-pr-0,eq:whpt-remask} together with~\Cref{lem:inflow-exact} we have \begin{equation*}
        \min\left(\Pr(X_t^i = b \mid x_{\ell}^{-i}), \wh \Pr(X_t^i = b \mid x_{\ell}^{-i})\right) \geq \nu(t, b) (1 - e^{-t}) \gtrsim \cF(t, b) (1 - e^{-t}),
    \end{equation*} and proceed with~\Cref{lem:ratio-ub} as follows:

    \begin{equation*}
    \begin{aligned}
        &\sum_{b\in\cV}\inflow(t, b)\left(\frac{\Pr(X_t^i = b \mid x_{\ell}^{-i})}{\wh \Pr(X_t^i = b \mid x_{\ell}^{-i})} - 1 - \log \frac{\Pr(X_t^i = b \mid x_{\ell}^{-i})}{\wh \Pr(X_t^i = b \mid x_{\ell}^{-i})} \right) \\
        &\quad \lesssim \sum_{b\in\cV}\frac{1}{1 - e^{-t}}\left(\Pr(X_t^i = b \mid x_{\ell}^{-i}) \log \frac{\Pr(X_t^i = b \mid x_{\ell}^{-i})}{\wh \Pr(X_t^i = b \mid x_{\ell}^{-i})} - \Pr(X_t^i = b \mid x_{\ell}^{-i}) + \wh \Pr(X_t^i = b \mid x_{\ell}^{-i})\right)
        \\ &\quad = \sum_{b\in\cV} \frac{1}{1 - e^{-t}} \left(\Pr(X_t^i = b \mid x_{\ell}^{-i}) \log \frac{\Pr(X_t^i = b \mid x_{\ell}^{-i})}{\wh \Pr(X_t^i = b \mid x_{\ell}^{-i})}\right).
    \end{aligned}
    \end{equation*}
    This implies that
    \begin{equation*}
    \cL_{\mathrm{approx}}^{(k)} \lesssim \int_u^{\ell} \sum_{i \in [d]} \bE_{x_{\ell}^{-i}} \sum_{b\in\cV} \frac{1}{1 - e^{-t}} \Pr(X_t^i = b \mid x_{\ell}^{-i}) \log \frac{\Pr(X_t^i = b \mid x_{\ell}^{-i})}{\wh \Pr(X_t^i = b \mid x_{\ell}^{-i})}.
    \end{equation*}
     Using again~\Cref{eq:pr-t-from-pr-0,eq:whpt-remask} and the convexity of the \(\KL\) divergence gives
    \begin{equation*}
        \sum_{b\in\cV}  \Pr(X_t^i = b \mid x_{\ell}^{-i}) \log \frac{\Pr(X_t^i = b \mid x_{\ell}^{-i})}{\wh \Pr(X_t^i = b \mid x_{\ell}^{-i})} \leq e^{-t} \sum_{b\in\cV}  \Pr(X_0^i = b \mid x_{\ell}^{-i}) \log \frac{\Pr(X_0^i = b \mid x_{\ell}^{-i})}{\wh \Pr(X_0^i = b \mid x_{\ell}^{-i})},
    \end{equation*}
    which shows
    \begin{equation*}
        \cL_{\mathrm{approx}}^{(k)} \lesssim \left(\int_u^{\ell} \frac{1}{e^t - 1}\d t\right) \sum_{i \in [d]}  \KL \left( \mu_{X_0^i} \ \big\|\ \wh\mu_{X_0^i} \ \Bigr|\  X_{\ell}^{-i}\,\right) \leq \frac{\ell - u}{e^u - 1} \sum_{i \in [d]}  \KL \left( \mu_{X_0^i} \ \big\|\ \wh\mu_{X_0^i} \ \Bigr|\  X_{\ell}^{-i}\,\right).
    \end{equation*}
    This concludes the proof.
\end{proof}

\subsubsection{Proof of~\Cref{lem:approx-to-se}}
    \Cref{lem:approx-exact} gives:
    \begin{equation*}
    \begin{aligned}
        &\cL_{\mathrm{approx}}^{(k)} = \bE_{x_{\ell} \sim q_{\ell}} \sum_{i\in[d]}  \int_u^{\ell} \sum_{b\in\cV} \Biggl[ \Pr(X_t^i = b \mid x_{\ell}^{-i}) \log \frac{\Pr(X_t^i = b \mid x_{\ell}^{-i})}{\wh \Pr(X_t^i = b \mid x_{\ell}^{-i})} \\
        &\qquad \qquad \qquad \qquad \qquad \qquad + \inflow(t, b)\left(\frac{\Pr(X_t^i = b \mid x_{\ell}^{-i})}{\wh \Pr(X_t^i = b \mid x_{\ell}^{-i})} - 1 - \log \frac{\Pr(X_t^i = b \mid x_{\ell}^{-i})}{\wh \Pr(X_t^i = b \mid x_{\ell}^{-i})} \right) \Biggr]\d t.
    \end{aligned}
    \end{equation*}

    \noindent
    \textbf{Uniform process.}
    We first prove the result for the uniform process, where \(\inflow(t, b) = \nicefrac{1}{S}\).
    The first term is the KL divergence, which is the Bregman divergence for \(\phi(\mu) = \sum_{b \in \cV} \mu(b) \log \mu(b)\), and the second term is the Bregman divergence for \(\phi(\mu) = -\sum_{b \in \cV} \log \mu(b)\). Here, \(\mu\) is a probability distribution over \(\cV\). Let \(t \in [u, \ell]\).
    For a general Bregman divergence \(\breg_{\phi}\) with \(\phi: \bR^n \to \bR\), using Hadamard's lemma, we have
    \begin{equation*}
        \breg_{\phi}(x, y) \defn \phi(x) - \phi(y) - (x - y)^\top \nabla \phi(y) = (x - y)^\top \left(\int_0^1 (1 - r) \nabla^2 \phi(y + r(x - y))\d r \right)(x - y).
    \end{equation*}
    For fixed \(t \in [u, \ell]\) and \(i \in [d]\) this gives, using \((x \log x)'' = \nicefrac 1 x\),
        \begin{equation*}
    \begin{aligned}
        f(t) &\defn \sum_{b \in \cV} \Pr(X_t^i = b \mid x_{\ell}^{-i}) \log \frac{\Pr(X_t^i = b \mid x_{\ell}^{-i})}{\wh \Pr(X_t^i = b \mid x_{\ell}^{-i})} \\
        &\quad = \sum_{b \in \cV}\left(\Pr(X_t^i = b \mid x_{\ell}^{-i}) - \wh \Pr(X_t^i = b \mid x_{\ell}^{-i})\right)^2 \int_0^1 \frac{1 - r}{r\Pr(X_t^i = b \mid x_{\ell}^{-i}) + (1 - r) \wh \Pr(X_t^i = b \mid x_{\ell}^{-i}) }\d r,
    \end{aligned}
    \end{equation*}
    and using \((-\log x)'' = \nicefrac 1 x^2\):
    \begin{equation*}
    \begin{aligned}
        g(t) &\defn \sum_{b \in \cV} \left(\frac{\Pr(X_t^i = b \mid x_{\ell}^{-i})}{\wh \Pr(X_t^i = b \mid x_{\ell}^{-i})} - 1 - \log \frac{\Pr(X_t^i = b \mid x_{\ell}^{-i})}{\wh \Pr(X_t^i = b \mid x_{\ell}^{-i})}\right) \\
        &\quad = \sum_{b \in \cV}\left(\Pr(X_t^i = b \mid x_{\ell}^{-i}) - \wh \Pr(X_t^i = b \mid x_{\ell}^{-i})\right)^2 \int_0^1 \frac{1 - r}{\left(r\Pr(X_t^i = b \mid x_{\ell}^{-i}) + (1 - r) \wh \Pr(X_t^i = b \mid x_{\ell}^{-i})\right)^2 }\d r.
    \end{aligned}
    \end{equation*}
    To upper bound \(f(t)\) with \(f(\ell)\), we need to upper bound, for fixed \(r \in [0, 1]\), 
    \begin{equation*}
        \frac{1}{r\Pr(X_t^i = b \mid x_{\ell}^{-i}) + (1 - r) \wh \Pr(X_t^i = b \mid x_{\ell}^{-i}) } \quad \text{with} \quad \frac{1}{r\Pr(X_{\ell}^i = b \mid x_{\ell}^{-i}) + (1 - r) \wh \Pr(X_\ell^i = b \mid x_{\ell}^{-i}) }.
    \end{equation*}
    Recall that \(\Pr(X_\ell^i = b \mid x_{\ell}^{-i}) = e^{-(\ell - t)}  \Pr(X_t^i = b \mid x_{\ell}^{-i}) + \frac{1}{S} (1 - e^{-(\ell - t)})\), similarly for \(\wh \Pr(X_\ell^i = b \mid x_{\ell}^{-i})\). Letting \(\alpha = \Pr(X_t^i = b \mid x_{\ell}^{-i})\) and \(\beta = \wh \Pr(X_t^i = b \mid x_{\ell}^{-i})\), we obtain:
    \begin{equation}
    \label{eq:comp-ft-with-fl}
    \begin{aligned}
        \frac{r\Pr(X_{\ell}^i = b \mid x_{\ell}^{-i}) + (1 - r) \wh \Pr(X_\ell^i = b \mid x_{\ell}^{-i})}{r\Pr(X_t^i = b \mid x_{\ell}^{-i}) + (1 - r) \wh \Pr(X_t^i = b \mid x_{\ell}^{-i}) } &= \frac{e^{-(\ell - t)} (r \alpha + (1 - r) \beta) +  \nicefrac {(1 - e^{-(\ell - t)})} S}{r \alpha + (1 - r) \beta} \\
        & = e^{-(\ell - t)} + \frac{\nicefrac {(1 - e^{-(\ell - t)})} S}{r \alpha + (1 - r) \beta} \\
        & \leq e^{-(\ell - t)} + \frac{ \nicefrac {(1 - e^{-(\ell - t)})} S }{ \nicefrac {(1 - e^{-t})} S} \\
        & = \frac{1 - e^{-\ell}}{1 - e^{-t}},
    \end{aligned}
    \end{equation}
    where the inequality follows as \(\alpha, \beta \geq \nicefrac { (1 - e^{-t})} S\). Since
    \begin{equation*}
        \left(\Pr(X_t^i = b \mid x_{\ell}^{-i}) - \wh \Pr(X_t^i = b \mid x_{\ell}^{-i})\right)^2 = e^{2(\ell - t)} \left(\Pr(X_\ell^i = b \mid x_{\ell}^{-i}) - \wh \Pr(X_\ell^i = b \mid x_{\ell}^{-i})\right)^2, 
    \end{equation*}
    we obtain
    \begin{equation*}
        \sum_{b \in \cV} \Pr(X_t^i = b \mid x_{\ell}^{-i}) \log \frac{\Pr(X_t^i = b \mid x_{\ell}^{-i})}{\wh \Pr(X_t^i = b \mid x_{\ell}^{-i})} \leq e^{\ell - t} \frac{e^\ell - 1}{e^t - 1} \sum_{b \in \cV} \Pr(X_\ell^i = b \mid x_{\ell}^{-i}) \log \frac{\Pr(X_{\ell}^i = b \mid x_{\ell}^{-i})}{\wh \Pr(X_\ell^i = b \mid x_{\ell}^{-i})}.
    \end{equation*}
    Squaring~\Cref{eq:comp-ft-with-fl}, we can also upper bound \(g(t)\) with \(g(\ell)\):
    \begin{equation*}
    \begin{aligned}
        &\sum_{b \in \cV} \left(\frac{\Pr(X_t^i = b \mid x_{\ell}^{-i})}{\wh \Pr(X_t^i = b \mid x_{\ell}^{-i})} - 1 - \log \frac{\Pr(X_t^i = b \mid x_{\ell}^{-i})}{\wh \Pr(X_t^i = b \mid x_{\ell}^{-i})}\right) \\
        &\quad \leq \left(\frac{e^\ell - 1}{e^t - 1}\right)^2 \sum_{b \in \cV} \left(\frac{\Pr(X_{\ell}^i = b \mid x_{\ell}^{-i})}{\wh \Pr(X_{\ell}^i = b \mid x_{\ell}^{-i})} - 1 - \log \frac{\Pr(X_{\ell}^i = b \mid x_{\ell}^{-i})}{\wh \Pr(X_{\ell}^i = b \mid x_{\ell}^{-i})}\right).
    \end{aligned}
    \end{equation*}
    Together with the bound \(\ell - u \leq \kappa \min(1, u)\) this proves that for the uniform process, \(\cL_{\mathrm{SE}}(t, \wh s_t, \wt s_t) \lesssim  \cL_{\mathrm{SE}}(\ell, \wh s_{\ell}, \wt s_{\ell})\), and as \(\wt s_\ell = s_\ell\), gives
    \begin{equation*}
    \cL_{\mathrm{approx}}^{(k)} \defn \bE_{x_{\ell} \sim q_{\ell}} \int_u^{\ell} \cL_{\mathrm{SE}}(t, \wh s_t, \wt s_t)\d t \lesssim (\ell - u) \cL_{\mathrm{SE}}(\ell, \wh s_\ell, s_\ell).
    \end{equation*}

    \noindent
    \textbf{Remasking process.} Now, we focus on the remasking process, where by~\Cref{lem:nu-remask} for \(b \in [S]\),
    \begin{equation*}
        \nu(t, b) = \frac{e^{-t}}{(1 - e^{-t}) S} \left(\cosh(t \sqrt{ 1  -p_M}) - 1\right).
    \end{equation*}
    Using that for all \(v > 0\) we have \(\Pr(X_v^i =b\mid x_{\ell}^{-i}) = e^{-v} \Pr(X_0^i=b \mid x_{\ell}^{-i}) + (1 - e^{-v}) \nu(v, b)\), we express
    \begin{equation*}
        \Pr(X_{\ell}^i = b \mid x_{\ell}^{-i}) = e^{-(\ell - t)} \Pr(X_t^i = b \mid x_{\ell}^{-i}) + (1 - e^{-\ell}) \nu(\ell, b) - (1 - e^{-t}) e^{-(\ell-t)} \nu(t, b),
    \end{equation*}
    and analogously for \(\wh \Pr(X_{\ell}^i = b \mid x_{\ell}^{-i})\).
    Repeating the steps from~\Cref{eq:comp-ft-with-fl}, we get (recall the notation \(\alpha = \Pr(X_t^i = b \mid x_{\ell}^{-i})\) and \(\beta = \wh \Pr(X_t^i = b \mid x_{\ell}^{-i})\))
    \begin{equation}
    \begin{aligned}
         \label{eq:up-b-approx}&\frac{r\Pr(X_{\ell}^i = b \mid x_{\ell}^{-i}) + (1 - r) \wh \Pr(X_\ell^i = b \mid x_{\ell}^{-i})}{r\Pr(X_t^i = b \mid x_{\ell}^{-i}) + (1 - r) \wh \Pr(X_t^i = b \mid x_{\ell}^{-i}) } \\
         &\quad = \frac{e^{-(\ell - t)} (r \alpha + (1 - r) \beta) + (1 - e^{-\ell}) \nu(\ell, b) - (1 - e^{-t})e^{-(\ell - t)} \nu(t, b)}{r \alpha + (1 - r) \beta} \\
        & = e^{-(\ell - t)} + \frac{(1 - e^{-\ell}) \nu(\ell, b) - (1 - e^{-t})e^{-(\ell - t)} \nu(t, b)}{r \alpha + (1 - r) \beta}.
    \end{aligned}
    \end{equation}
    If \(p_M = 1\), we have that \(\nu(t, b) = 0\) for \(b \in [S]\) and all \(t \geq 0\), thus the latter expression equals \(e^{-(\ell - t)} \leq 1\). When \(p_M < 1\), we continue~\Cref{eq:up-b-approx} as follows:
    \begin{equation}
        \begin{aligned}
        &\frac{r\Pr(X_{\ell}^i = b \mid x_{\ell}^{-i}) + (1 - r) \wh \Pr(X_\ell^i = b \mid x_{\ell}^{-i})}{r\Pr(X_t^i = b \mid x_{\ell}^{-i}) + (1 - r) \wh \Pr(X_t^i = b \mid x_{\ell}^{-i}) } \\
        & \leq e^{-(\ell - t)} + \frac{(1 - e^{-\ell}) \nu(\ell, b) - (1 - e^{-t})e^{-(\ell - t)} \nu(t, b)}{(1 - e^{-t}) \nu(t, b)} \\
        & = \frac{(1 - e^{-\ell}) \nu(\ell, b)}{(1 - e^{-t})\nu(t, b)},
    \end{aligned}
    \end{equation}
    Using the expression for \(\nu(t, b)\) for the remasking process, we obtain
    \begin{equation*}
    \frac{(1 - e^{-\ell}) \nu(\ell, b)}{(1 - e^{-t})\nu(t, b)} = \frac{e^{-\ell}\left(\cosh(\ell \sqrt{1 - p_M}) - 1\right)}{e^{-t}\left(\cosh(t \sqrt{1 - p_M}) - 1\right)}.
    \end{equation*}
    As the ratio is strictly decreasing with respect to \(p_M\), we consider the case \(p_M = 0\), which gives
    \begin{equation*}
        \frac{(1 - e^{-\ell}) \nu(\ell, b)}{(1 - e^{-t})\nu(t, b)} = \left(\frac{1 - e^{-\ell}}{1 - e^{-t}}\right)^2 \leq (1 + \kappa)^2 \lesssim 1.
    \end{equation*}
    The rest of the proof follows closely the argument for the uniform process. This concludes the proof.

\section{Proofs of results in~\Cref{sec:technical}}
\subsection{Proof of~\Cref{prop:it-results}}
    The first property follows immediately from the definition of the mutual information, the total correlation, and the dual total correlation. The second and third properties follow as:
\begin{equation}
\label{eq:mi-der}
    \begin{aligned}
        \frac{\partial}{\partial v} \ent(X_t^i \mid  X_v^{-i})
        &= -\bE_{x_t^i, x_v^{-i}} \left[\left(Q^{-i}_v \log \Pr(x_t^i \mid \cdot\ )\right)(x_v^{-i})\right], \\ \frac{\partial}{\partial t} \ent(X_t^i \mid  X_v^{-i})
        &= -\bE_{x_t^i, x_v^{-i}} \left[\left(Q^{i}_t \log \Pr(\, \cdot \mid x_v^{-i}\,)\right)(x_t^i)\right], \quad \text{and} \\
        \mi(X_t^i\ ;\, X_v^{-i}) &= \ent(X_t^i) - \ent(X_t^i \mid X_v^{-i}).
     \end{aligned}
\end{equation}
For the fourth property, we proceed as follows:
    
    \begin{equation}
    \label{eq:ent}
    \begin{aligned}
    \frac{\d}{\d v} \ent(X_v) &=
        \frac{\d }{\d v} \left(\bE_{x_v} \log \frac{1}{\Pr(X_v =x_v)}\right) \\
        &= \sum_{x_v\in\cV^{d}} \left(\frac{\partial}{\partial v}\Pr(X_v = x_v) \right)\log \frac{1}{\Pr(X_v =x_v)} \\
        &= \bE_{x_v} \sum_{i\in[d]} \sum_{b\in\cV} \qtok_v(x_v^i, b) \log \frac{\Pr(X_v = x_v)}{\Pr(X_v = x_v \odot_i b)}\\
        &= \bE_{x_v} \sum_{i\in[d]} \sum_{b\in\cV} \qtok_v(x_v^i, b) \log \frac{\Pr(X_v^i = x_v^i \mid X_v^{-i} = x_v^{-i})}{\Pr(X_v^i = b \mid X_v^{-i} = x_v^{-i})}.
    \end{aligned}
    \end{equation}
    Similarly, it can be easily calculated that 
    \begin{equation}
    \label{eq:cont-ent}
    \begin{aligned}
        &\frac{\d }{\d v}  \ent(X_v^i \mid X_v^{-i}) \\
        &\quad = \bE_{x_v} \sum_{j\in[d]} \sum_{b\in\cV}\qtok_v(x_v^j, b) \log \frac{\Pr(X_v^i = x_v^i \mid X_v^{-i} = x_v^{-i}) }{\Pr(X_v^i = (x_v \odot_j b)^i \mid X_v^{-i} = (x_v\odot_j b)^{-i})} \\
        &\quad = \bE_{x_v} \sum_{b\in\cV} \qtok_v(x_v^i, b) \log \frac{\Pr(X_v^i = x_v^i \mid X_v^{-i} = x_v^{-i}) }{\Pr(X_v^i = b \mid X_v^{-i} = x_v^{-i})} + \bE_{x_v} \sum_{j\neq i} \sum_{b\in\cV} \qtok_v(x_v^j, b) \log \frac{\Pr(X_v^i = x_v^i \mid X_v^{-i} = x_v^{-i}) }{\Pr(X_v^i = x_v^i \mid X_v^{-i} = x_v^{-i}\odot_j b)}.
    \end{aligned}
    \end{equation}
    Taking these collectively yields 
    \begin{align*}
        \frac{\d}{\d v} \dtc(X_v) &= \frac{\d}{\d v} \left(\ent(X_v) - \sum_{i\in[d]} \ent(X_v^i \mid X_v^{-i})\right) \\
        & = -\bE_{x_v} \sum_{i \in [d]} \sum_{j\neq i} \sum_{b\in\cV} \qtok_v(x_v^j, b) \log \frac{\Pr(X_v^i = x_v^i \mid X_v^{-i} = x_v^{-i}) }{\Pr(X_v^i = x_v^i \mid X_v^{-i} = x_v^{-i}\odot_j b)} \\
        & = \bE_{x_v} \sum_{i \in [d]} \left[\left(Q^{-i}_v \log \Pr(x_v^i \mid \cdot\ )\right)(x_v^{-i})\right]\\
        & = \sum_{i \in [d]}\frac{\partial}{\partial v} \mi(X_t^i\ ;\, X_v^{-i}) \biggr\rvert_{t = v},    \end{align*}
    which, together with (i), proves (iv). For the last property, we continue from (ii):
\begin{align}
\label{eq:mi-second-der-together}
    \notag &\frac{\partial^2 }{\partial t\partial v} \mi(X_t^i\,;\, X_v^{-i}) \\
    \notag & = \bE_{x_v^{-i}} \frac{\partial}{\partial t} \left(\bE_{x_t^i \mid x_v^{-i}} \sum_{j \neq i} \sum_{b\in\cV} \qtok_v(x_v^j, b)  \log \frac{\Pr(x_t^i \mid x_v^{-i} \odot_j b)}{\Pr(x_t^i \mid x_v^{-i})}\right) \\
    \notag& = \bE_{x_t^i, x_v^{-i}} \sum_{j \neq i} \sum_{b, c\in\cV} \qtok_v(x_v^j, b) \qtok_t(x_t^i, c)  \left(\log \frac{\Pr(X_t^i = c \mid x_v^{-i} \odot_j b)}{\Pr(X_t^i = c \mid x_v^{-i})} - \log \frac{\Pr(x_t^i \mid x_v^{-i} \odot_j b)}{\Pr(x_t^i \mid x_v^{-i})}\right) \\
    & \quad +  \bE_{x_t^i, x_v^{-i}} \sum_{j \neq i} \sum_{b,c\in\cV} \qtok_v(x_v^j, b) \qtok_t(c, x_t^i) \left(\frac{ \Pr(X_t^i = c \mid x_v^{-i} \odot_j b)}{\Pr(x_t^i \mid x_v^{-i} \odot_j b)} - \frac{\Pr(X_t^i = c \mid x_v^{-i})}{\Pr(x_t^i \mid x_v^{-i})}\right). 
\end{align}
For fixed \(x_v^{-i}, j, b\), recall \(r(a) = \frac{\Pr(X_t^i = a \mid X_v = x_v^{-i})}{\Pr(X_t^i = a \mid X_v = x_v^{-i} \odot_j b)}\). Then, we have
\begin{equation}
\begin{aligned}
\label{eq:mi-second-der-first}
    &\bE_{x_t^i, x_v^{-i}} \left[\log \frac{\Pr(X_t^i = c \mid x_v^{-i} \odot_j b)}{\Pr(X_t^i = c \mid x_v^{-i})} - \log \frac{\Pr(x_t^i \mid x_v^{-i} \odot_j b)}{\Pr(x_t^i \mid x_v^{-i})}\right] = \bE_{x_v^{-i}} \sum_{x_t^{-i}\in\cV} \Pr(x_t^i \mid x_v^{-i} \odot_j b) r(x_t^i) \log \frac{r(x_t^i)}{r(c)}, 
\end{aligned}
\end{equation}
and 
\begin{equation}
    \begin{aligned}
    \label{eq:mi-second-der-second}
        &\bE_{x_t^i, x_v^{-i}} \sum_{c\in\cV} \qtok_t(c, x_t^i) \left(\frac{ \Pr(X_t^i = c \mid x_v^{-i} \odot_j b)}{\Pr(x_t^i \mid x_v^{-i} \odot_j b)} - \frac{\Pr(X_t^i = c \mid x_v^{-i})}{\Pr(x_t^i \mid x_v^{-i})}\right) \\
        & \quad = \bE_{x_v^{-i}} \sum_{x_t^i, c\in\cV} \qtok_t(c, x_t^i) \Pr(x_t^i \mid x_v^{-i}) \left(\frac{ \Pr(X_t^i = c \mid x_v^{-i} \odot_j b)}{\Pr(x_t^i \mid x_v^{-i} \odot_j b)} - \frac{\Pr(X_t^i = c \mid x_v^{-i})}{\Pr(x_t^i \mid x_v^{-i})}\right) \\
        & \quad = \bE_{x_v^{-i}} \sum_{x_t^i, c\in\cV} \qtok_t(x_t^i, c) \Pr(X_t^i = c \mid x_v^{-i}) \left(\frac{ \Pr(x_t^i \mid x_v^{-i} \odot_j b)}{\Pr(X_t^i = c \mid x_v^{-i} \odot_j b)} - \frac{\Pr(x_t^i \mid x_v^{-i})}{\Pr(X_t^i = c \mid x_v^{-i})}\right) \\
        & \quad = \bE_{x_v^{-i}} \sum_{x_t^i, c\in\cV} \qtok_t(x_t^i, c) \Pr(x_t^i \mid x_v^{-i} \odot_j b) \left(\frac{ \Pr(X_t^i = c \mid x_v^{-i})}{\Pr(X_t^i = c \mid x_v^{-i} \odot_j b)} - \frac{\Pr(x_t^i \mid x_v^{-i})}{\Pr(x_t^i \mid x_v^{-i} \odot_j b)}\right) \\
        & \quad = \bE_{x_v^{-i}} \sum_{x_t^i, c\in\cV} \qtok_t(x_t^i, c) \Pr(x_t^i \mid x_v^{-i} \odot_j b) \left(r(c) - r(x_t^i)\right),
    \end{aligned}
\end{equation}
where we relabeled \(x_t^i \leftrightarrow c\) in the third line.
Plugging~\Cref{eq:mi-second-der-first,eq:mi-second-der-second} into~\Cref{eq:mi-second-der-together} and 
setting \(a = x_t^i\), we obtain
\begin{equation}
\label{eq:second-par-der}
    \frac{\partial^2 }{\partial t\partial v} \mi(X_t^i\ ;\, X_v^{-i}) = \bE_{x_v^{-i}} \sum_{j \neq i} \sum_{a, b, c}\qtok_v(x_v^j, b) \qtok_t(a, c)\Pr(X_t^i = a \mid x_v^{-i} \odot_j b) r(c) \left[\frac{r(a)}{r(c)} \log \frac{r(a)}{r(c)} + 1- \frac{r(a)}{r(c)} \right] \geq 0,
\end{equation}
as \(x \log x + 1 - x \geq 0\) for all \(x \geq 0\).

\subsection{Proof of~\Cref{prop:cond-exp}}
As \(\wt Q_{T-t}(x \odot_i x_t^i, x \odot_i b) = \qtok_t(b, x_t^i) \wt s_t(x \odot_i b, x \odot_i x_t^i)\) and \(\back Q_{T-t}(x_t, x_t \odot_i b) = \qtok_t(b, x_t^i) s_t(x_t \odot_i b, x_t)\), we proceed by showing 
\begin{equation*}
    \wt s_t(x \odot_i b, x \odot_i x_t^i) =
        \bE_{x_t^{-i}}\Bigl[s_t(x_t \odot_i b, x_t) \Bigr\rvert X_{\ell}^{-i} = x_{\ell}^{-i}, X_t^i = x_t^i\Bigr].
\end{equation*}
    Here and below we omit stochastic process notation, such as \(X_t = \cdot\), \(X_{\ell}^{-i} = \cdot\), etc., when the time and dimension indices are clear from the context. We have
        \begin{align*}
         \bE_{x_t^{-i}}\Bigl[s_t(x_t \odot_i b, x_t) \Bigr\rvert X_{\ell}^{-i} = x_{\ell}^{-i}, X_t^i = x_t^i\Bigr] & =\bE\left[\frac{\Pr(x_t \odot_i b)}{\Pr(x_t)}\mid x_\ell^{-i}, x_t^i\right] \\
         &= \sum_{x_t^{-i}\in\cV^{d-1}} \frac{\Pr(x_t \odot_i b)}{\Pr(x_t)} \Pr\left(x_t^{-i}\ \big|\ x_\ell^{-i}, x_t^i\right) \\
        & = \sum_{x_t^{-i}\in\cV^{d-1}} \frac{\Pr(x_t \odot_i b)}{\Pr(x_t)} \cdot \frac{\Pr\left(x_\ell^{-i}\ \big|\ x_t^{-i}\right) \Pr\left(x_t^{-i}\ \big|\  x_t^i\right)}{\Pr\left(x_\ell^{-i}\ \big|\  x_t^i\right)} \\
        & = \sum_{x_t^{-i}\in\cV^{d-1}} \frac{\Pr(x_t \odot_i b)}{\Pr(x_t^i) \cancel{\Pr(x_t^{-i} \mid x_t^i)}} \cdot \frac{\Pr(x_\ell^{-i} \mid x_t^{-i}) \cancel{\Pr(x_t^{-i} \mid x_t^i)}}{\Pr(x_\ell^{-i} \mid x_t^i)}. \\
    \end{align*}
    Observe that in the denominator we have 
        \(\Pr(x_t^i)\Pr(x_\ell^{-i} \mid x_t^i)= \Pr(x_{\ell}^{-i},\, x_t^i) = \sum_{x_t^{-i}\in\cV^{d-1}} \Pr(x_t) \Pr(x_{\ell}^{-i} \mid x_t^{-i})\). 
    We continue:
    \begin{align*}
        \bE_{x_t^{-i}}\Bigl[s_t(x_t \odot_i b, x_t) \Bigr\rvert X_{\ell}^{-i} = x_{\ell}^{-i}, X_t^i = x_t^i\Bigr] & = \frac{\sum_{x_t^{-i}} \Pr(x_t \odot_i b) \Pr(x_\ell^{-i} \mid x_t^{-i})}{\sum_{x_t^{-i}} \Pr(x_t) \Pr(x_\ell^{-i} \mid x_t^{-i})} \\
        & = \frac{\sum_{x_t^{-i}} \Pr(x_t \odot_i b) \Pr(x_\ell^{-i} \mid x_t^{-i}) / \Pr(x_\ell^{-i})}{\sum_{x_t^{-i}} \Pr(x_t) \Pr(x_\ell^{-i} \mid x_t^{-i}) / \Pr(x_\ell^{-i})} \\
        & = \frac{\sum_{x_t^{-i}} \Pr(x_t \odot_i b \mid x_{\ell}^{-i})}{\sum_{x_t^{-i}} \Pr(x_t \mid x_{\ell}^{-i})} \\
        & = \frac{\Pr(X_t^i = b \mid x_{\ell}^{-i})}{\Pr(X_t^i = x_t^i \mid x_{\ell}^{-i})}.
        \end{align*}
    Recalling that by~\Cref{def:wt-score}, \(\wt s_t(x \odot_i b, x \odot_i x_t^i) =\Pr(X_t^i = b \mid x_{\ell}^{-i})\ /\,\Pr(X_t^i = x_t^i \mid  x_{\ell}^{-i})\) finishes the proof.

\subsection{Proof of~\Cref{lem:ratio-ub}}
    First, assume that \(q < p\). In this case, as 
    \begin{equation*}
        x - 1 - \log x \leq x \log x - x + 1 \qquad \text{for} \quad x \geq 1,
    \end{equation*}
    we have
    \begin{equation*}
        \frac{\frac p q  - 1 - \log \frac p q}{q\left(\frac p q \log \frac p q - \frac p q + 1 \right)}  \leq \frac 1 q =  \frac 1 {\min(p, q)} \leq \frac 1 \alpha. 
    \end{equation*}
    The case \(p < q\) follows similarly as
    \begin{equation*}
        x(x - 1 - \log x) \leq x \log x - x + 1 \qquad \text{for} \quad x \leq 1.
    \end{equation*}
    Finally, as \(
        \lim_{x \to 1} \frac{x-  1 - \log x}{x\log x - x + 1} = 1\),
    the case \(p =q\) follows by continuity.

\subsection{Proof of~\Cref{lem:sum-prob}}
Recall the definition of \(\nu(t, a)\) from~\Cref{eq:def-nu}. We have
\begin{equation*}
\begin{aligned}
    \Pr(X_t^i = a) &= \Pr(X_0^i = a) e^{-t} + \nu(t, a)(1 - e^{-t}), \text{ and} \\
     \Pr(X_t^i = a \mid X_0^i = c) &= \bI\{a = c\} e^{-t} + \nu(t, a)(1 - e^{-t}).
\end{aligned}
\end{equation*}
Taking the difference and summing over all \(a \in \cV\), we get
\begin{equation*}
\begin{aligned}
    \sum_{a\in\cV} \qtok(a, b) \left(\Pr(X_t^i = a \mid X_0^i = c) - \Pr(X_t^i = a) \right) &=  \left(\sum_{a\in\cV} \qtok(a, b) \bI\{a = c\} - \sum_{a\in\cV} \qtok(a, b) \Pr(X_0^i = a) \right) e^{-t} \\
    & = \left(\qtok(c,b) - \sum_{a\in\cV} \qtok(a, b) \Pr(X_0^i = a) \right) e^{-t}. 
\end{aligned}
\end{equation*}
For the uniform process, the term in the bracket equals zero as all \(\qtok(a, b) = \nicefrac{1}{S}\). For the remasking process, as \(c \in [S]\) we consider three cases: (i) \(b \in [S]\), (ii) \(b = \mathrm{REMASK}\), and (iii) \(b = \mask\). In the first case, we have \(\qtok(c, b) = 0\) and the only \(a \in \cV\) such that \(\qtok(a, b) \neq 0\) is \(a = \mathrm{REMASK}\), for which \(\Pr(X_0^i = a) = 0\). In the second case, \(\qtok(c, b) = 1 - p_M = \sum_{a \in \cV} \qtok(a, b) \Pr(X_0^i = a)\). The third case follows equivalently. This proves~\Cref{eq:indep-initial}.  
Next,  we have
    \begin{equation*}
    \begin{aligned}
        &\sum_{a\in\cV} \qtok(a, b) \Pr(X_t^i = a \mid X_{u}^{-i} = x^{-i}) \\
        &\quad  = \sum_{x_0^i\in\cV} \Pr(X_0^i = x_0^i \mid X_u^{-i} = x^{-i}) \left(\sum_{a\in\cV} \qtok(a, b) \Pr(X_t^i = a \mid X_{0}^{i} = x_0^{i})\right) \\
        & \quad = \inflow(t, b) \sum_{x_0^i\in\cV} \Pr(X_0^i = x_0^i \mid X_u^{-i} = x^{-i}) \\
        &\quad = \inflow(t, b),
    \end{aligned}
    \end{equation*}
    which proves~\Cref{eq:indep-second}. Finally, using~\Cref{eq:whpr-eq-pr}, we can replace \(\Pr(\cdot \mid \cdot)\) with \(\wh \Pr(\cdot \mid \cdot)\) in~\Cref{eq:indep-initial,eq:indep-second}, which concludes the proof.
\section{Connection to the effective total correlation}

In~\cite{dmitriev26efficient}, for the masking noising process, the effective total correlation is studied:
\begin{equation*}
    \cD(q_{\mathrm{data}}) \defn \int_0^{\infty} \min(1, t) \sum_{i \neq j} \mi(X_t^i\ ;\,  X_t^j \mid X_t^{-i,j}) \d t.
\end{equation*}
The next proposition connects this quantity with the presented results.
\begin{proposition}
    For the masking noising process,
    \begin{equation*}
        \int_0^1 \dtc(X_t) \d t =  \cD(q_{\mathrm{data}}).
    \end{equation*}
\end{proposition}
\begin{proof}
    We show that both expressions are equal to
    \begin{equation*}
        -\int_0^{\infty} \min(1, t) \frac{\d}{\d t}\dtc(X_t)\d t.
    \end{equation*}
    Indeed,
    \begin{equation*}
    \begin{aligned}
        -\int_0^{\infty} \min(1, t) \frac{\d}{\d t}\dtc(X_t)\d t &= -\int_0^1 t \frac{\d}{\d t}\dtc(X_t)\d t - \int_1^{\infty}  \frac{\d}{\d t}\dtc(X_t)\d t \\
        & = t \dtc(X_t) \biggr\rvert_{t=1}^{t=0} + \int_0^1 \dtc(X_t) \d t - \dtc(X_t)\biggr\rvert_{t=\infty}^{t=1} \\
        & = \int_0^1 \dtc(X_t) \d t.
    \end{aligned}
    \end{equation*}
    Furthermore, observe that as \(\mi(X_t^i\ ;\, X_v^{-i}) = \mi(X_t^i\ ;\,X_v^{-(i,j)}) + \mi(X_t^i\ ;\, X_v^j \mid X_v^{-i,j})\) for any \(j \neq i\), we have
    \begin{equation*}
        \frac{\partial}{\partial v} \mi(X_t^i\ ;\, X_v^{-i}) = \sum_{j \neq i} \frac{\partial}{\partial v_j} \mi(X_t^i\ ;\, X_v^{-i}) = \sum_{j \neq i} \frac{\partial}{\partial v_j} \mi(X_t^i\ ;\, X_v^j \mid X_v^{-i,j}),
    \end{equation*}
    where \(\frac \partial {\partial v_j}\) denotes taking the derivative only with respect to the \(j\)-th coordinate of \(X_v^{-i}\).
    Next, for the masking noising process,
    \begin{equation*}
        \frac{\partial}{\partial v_j} \mi(X_t^i\ ;\, X_v^j \mid X_v^{-i,j}) = \mi(X_t^i\ ;\, \mask \mid X_v^{-i,j}) -\mi(X_t^i\ ;\, X_v^j \mid X_v^{-i,j})=  -\mi(X_t^i\ ;\, X_v^j \mid X_v^{-i,j}).
    \end{equation*}
    Using that \(
        \sum_i\frac{\partial}{\partial v} \mi(X_t^i\ ;\, X_v^{-i}) \biggr\rvert_{t = v} = \frac{\d}{\d v} \dtc(X_v)\), as shown in~\Cref{prop:it-results} (iv), we obtain
    \begin{equation*}\int_0^{\infty} \min(1, t) \sum_{i \neq j} \mi(X_t^i\ ;\,  X_t^j \mid X_t^{-i,j}) \d t =
        -\int_0^{\infty} \min(1, t) \frac{\d}{\d t}\dtc(X_t)\d t,
    \end{equation*}
    which finishes the proof.
\end{proof}
For the masking process, as \(\cF(t, b) = 0\) for \(b \in [S]\), one can show a stronger version of~\Cref{cor:second-der-ub}:
\begin{equation}
    \frac{\partial^2}{\partial t \partial v} \mi(X_t^i\ ;\, X_v^{-i}) = -\frac{\partial}{\partial v} \mi(X_t^i\ ;\, X_v^{-i}),
\end{equation}
which implies that the discretization error can be written as
\begin{equation}
    \sum_{k=0}^{N - 1} \cL_{\mathrm{discr}}^{(k)} \lesssim - \kappa \int_0^{\infty} \min(1, t) \frac{\d}{\d t}\dtc(X_t)\d t = \kappa \cD(q_{\mathrm{data}}).
\end{equation}
Therefore, our results recover exactly the best previously known bound for the masking noising process and obtain the first adaptive bound for the uniform and remasking noising processes.

\bibliographystyle{apalike}
\bibliography{references}

\end{document}